\documentclass[10pt, journal,web]{article}

\usepackage{amsthm}

\newtheorem{theorem}{Theorem}

\newtheorem{proposition}{Proposition}
\newtheorem{corollary}{Corollary}

\usepackage{times}
\usepackage{subcaption}
\usepackage[utf8]{inputenc}
\usepackage{utfsym}
\usepackage{multirow}
\usepackage[normalem]{ulem}
\usepackage{amsmath,amssymb,amsfonts}
\usepackage{algorithmic}
\useunder{\uline}{\ul}{}
\usepackage{url}
\def\BibTeX{{\rm B\kern-.05em{\sc i\kern-.025em b}\kern-.08em
    T\kern-.1667em\lower.7ex\hbox{E}\kern-.125emX}}

\usepackage{graphicx}
\usepackage{authblk}

\usepackage{longtable}
\usepackage{pdflscape}
\usepackage[T1]{fontenc}
\usepackage{bbm}
\usepackage[ruled,vlined]{algorithm2e}
\usepackage{pifont}
\usepackage{xcolor}
\usepackage{hyperref}
\usepackage{booktabs}
\hypersetup{
   colorlinks=true,
    linkcolor=blue,
   citecolor=blue,
    urlcolor=blue
}
\usepackage{makecell}
\usepackage{tabularx}
\usepackage{enumitem}
\usepackage{appendix}

\makeatletter
\renewcommand\AB@affilsepx{, \protect\Affilfont}
\makeatother

\providecommand{\keywords}[1]{%
  \small
  \textbf{\textit{Keywords---}} #1%
}

\begin{document}

\title{\textbf{Homophily-aware Supervised Contrastive Counterfactual Augmented Fair Graph Neural Network}
}
\author[1]{M. Tavassoli Kejani}
\author[2, 3]{F. Dornaika\thanks{Corresponding author}}
\author[4]{C. Laclau}
\author[1,5]{J. M. Loubes}
\affil[1]{\textit{Institut de Mathématiques de Toulouse}}
\affil[2]{\textit{University of the Basque Country}}
\affil[3]{\textit{IKERBASQUE}}
\affil[4]{\textit{Institut Polytechnique de Paris}}
\affil[5]{\textit{INRIA, Projet Regalia}}

\affil[ ]{

\small\texttt{mahdi.tavassoli-kejani@univ-toulouse.fr, fadi.dornaika@ehu.eus, charlotte.laclau@telecom-paris.fr, loubes@math.univ-toulouse.fr}}
\date{}

\maketitle
\begin{abstract}
In recent years, Graph Neural Networks (GNNs) have achieved remarkable success in tasks such as node classification, link prediction, and graph representation learning. However, they remain susceptible to biases that can arise not only from node attributes but also from the graph structure itself. Addressing fairness in GNNs has therefore emerged as a critical research challenge. In this work, we propose a novel model for training fairness-aware GNNs by improving the counterfactual augmented fair graph neural network framework (CAF). Specifically, our approach introduces a two-phase training strategy: in the first phase, we edit the graph to increase homophily ratio with respect to class labels while reducing homophily ratio with respect to sensitive attribute labels; in the second phase, we integrate a modified supervised contrastive loss and environmental loss into the optimization process, enabling the model to jointly improve predictive performance and fairness. Experiments on five real-world datasets demonstrate that our model outperforms CAF and several state-of-the-art graph-based learning methods in both classification accuracy and fairness metrics.
\end{abstract}

 \hspace{10pt}

\section{Introduction}
Graph Neural Networks (GNNs) have emerged as a powerful framework for learning on graph-structured data, achieving state-of-the-art performance in node classification, link prediction, and graph classification. By aggregating information, from the neighbors of the node, GNNs capture both local and global structural properties. Most existing GNNs, such as Graph Convolutional Networks (GCNs), rely on the assumption of graph homophily (or edge homophily) \cite{zhu2020beyond}, where the connected nodes tend to share similar features or belong to the same class. However, many real-world graphs instead exhibit heterophily, where nodes with dissimilar features or labels are more likely to be connected, for example, to different amino acid types in protein structures. Despite these challenges, GNNs have been successfully applied in domains such as social networks, recommendation systems, and drug discovery. For a comprehensive recent survey of GNNs and applications, see \cite{Gkarmpounis2024}.
At the same time, the relational nature of GNNs makes them particularly sensitive to biases encoded in node attributes or connectivity patterns, which can amplify unfairness in downstream predictions \cite{chen-survey}. 
For example, medical decision-support systems that use patient-similarity graphs may detect characteristic symptoms of cardiovascular disease less accurately in female patients than in male patients. Consequently, these models may perform better for one demographic group than another, leading to unequal treatment recommendations and unfair medical outcomes. Ensuring that GNN-based models generate fair, unbiased, and reliable output is essential, particularly in high-stakes applications.
Fairness in machine learning has received growing attention in recent years, with methods generally seeking to limit the influence of  attributes, such as gender, race, or age, on model predictions when disparities cannot be ethically, legally or socially justified \cite{chouldechova2020snapshot, 3495724.3497012, barocas2017fairness, besse2019can, besse2021survey}. In this work such variables which model an information that could bias the output of a machine learning model, will be referred to as sensitive attributes. We focus our theory on binary variables $S=0$ stands for a minority group while $S=1$ refers to the majority group. This binary case is not restrictive since many rights-violating practices induce a two-way split between affected and unaffected individuals.
Extending this line of work, approaches to fair GNNs can be broadly categorized into four groups \cite{chen2024fairness}:
\begin{enumerate}
\item \textbf{Pre-processing methods}, which mitigate bias before training GNNs (e.g., \cite{spinelli2021fairdrop});
\item \textbf{In-training methods}, which address bias during model optimization (e.g., \cite{liu2023generalized});
\item \textbf{Post-processing methods}, which remove bias from the learned GNN representations (e.g., \cite{kose2023fairness}). This line of work is closely related to counterfactual fairness \cite{kusner2017counterfactua, de2021transport}, which formalizes the principle that a decision should remain unchanged in a hypothetical world where only the sensitive attribute label differs.
\item \textbf{Hybrid methods}, which combine two or more of the above strategies (e.g., \cite{kang2020inform}).
\end{enumerate}
Recently, methods have aimed to mitigate the root cause of bias by modeling the causal relationships among variables. The identified causal structure allows for the adjustment of sensitive data to generate counterfactuals, ensuring that predictions remain unaffected by sensitive attributes. Examples of such methods include NIFTY \cite{agarwal2021towards}, GEAR \cite{ma2022learning}, CAF \cite{guo2023towards}, and FairGB \cite{li2024rethinking}.

The Counterfactual Augmented Fair Graph Neural Network (CAF) \cite{guo2023towards} is a recent in-training fairness method that constructs a causal model based on realistic counterfactuals, rather than perturbing sensitive attribute labels as in NIFTY \cite{agarwal2021towards}. To achieve fairness, CAF requires access to fully supervised sensitive attribute labels and semi-supervised class labels. However, it needs knowledge of all labels of unlabeled nodes, which are obtained through a pre-training phase by GNNs model for generating realistic counterfactuals. The model is then re-trained with additional fairness constraints.

In CAF, node representations are split into two equal parts: one encoding information related exclusively to the class label (content representation), and the other encoding the sensitive attribute (environmental representation). This separation is intended to prevent sensitive information from leaking into the content representation and vice versa. However, in many cases, as shown in Figure~\ref{fig:tsne-german} and \ref{fig:tsne-bail}, these constraints are insufficient to fully disentangle content and environmental representations. To address this limitation, we extend CAF by introducing two new constraints, modified supervised contrastive loss and environmental loss.

A key challenge arises from the so-called \emph{topology bias} \cite{jiang2022topology}: in many graphs, such as social networks \cite{mcpherson2001birds} and citation networks \cite{ciotti2016homophily}, nodes with similar sensitive attribute labels are more likely to connect than nodes with different ones. The creation of such groups of nodes, driven more by sensitive attribute labels than by class labels, can affect message passing in GNNs, and thereby amplify model bias. This bias can be quantified via the \emph{homophily ratio} with respect to sensitive attribute labels, defined as the proportion of edges connecting nodes that share the same sensitive attribute labels. A higher homophily ratio indicates stronger bias, as discussed in  \cite{loveland2025unveiling}, \cite{laclau2022survey}, \cite{chen2024fairness} and references therein. Our method mitigates topology bias by decreasing homophily with respect to sensitive attribute labels while preserving the class label information conveyed by the graph by maintaining homophily with respect to class labels.  Since CAF’s pre-training phase requires access to all unlabeled nodes, we also propose a \emph{graph editing strategy} based on homophily to reduce topology bias as a pre-processing step before applying other constraints. The experimental results, measured using balanced accuracy, demonstrate consistent improvements over existing methods.

In summary, we propose a hybrid approach that extends the CAF framework by combining pre-processing and in-training techniques. Specifically, we introduce a graph editing strategy in the pre-processing phase and introduce two new training objectives (modified supervised contrastive loss and environmental loss) designed to preserve task-relevant information while promoting fairness.

The main contributions of this paper are as follows.
\begin{enumerate}
    \item We propose a \emph{graph editing strategy} to mitigate topology bias with respect to sensitive attribute labels while preserving the class label information.
    \item We introduce a modified \emph{Supervised Contrastive Loss} by Student-t von Mises-Fisher that encourages nodes with the same label to have similar content representations, regardless of sensitive attribute labels.
    \item We define an \emph{Environmental Loss} that facilitates  disentanglement by penalizing the similarity between the environmental representations of nodes with differing sensitive attribute labels.
    \item We conduct extensive experiments on five real-world datasets, with particular attention to \emph{balanced accuracy} as a utility measure, an aspect often overlooked in prior work.
\end{enumerate}
The remainder of this paper is organized as follows. Section~2 reviews recent methods for fair GNNs. Section~3 introduces the CAF framework. Section~4 presents our proposed extensions, including graph editing, Supervised Contrastive Loss, and Environmental Loss. Section~5 reports experimental results and analysis. Finally, Section~6 concludes and discusses directions for future work.

\section{Related Works}
In this article, we focus on bias and graph-based models. Existing approaches adapt classical fairness techniques to the graph domain, either by modifying model training or through post-processing methods. Among the most popular methods, we can mention FairGNN \cite{dai2021say} which uses adversarial debiasing to enforce independence between node embeddings and sensitive attributes. Similarly, FairVGNN \cite{wang2022improving} enhances fairness by masking sensitive-correlated feature channels and clamping weights to reduce the influence of sensitive information, with the support of adversarial discriminators. In addition, FairDrop \cite{spinelli2021fairdrop} proposes an edge-dropout algorithm to improve fairness and reduce homophily in graph representation learning. Similarly, FairWalk \cite{rahman2019fairwalk} incorporates fairness into random-walk sampling for node embedding, and EDITS \cite{dong2022edits} addresses bias in both node features and graph topology via a data repair strategy. Other work draws on tools such as optimal transport to design fair representations in structured settings \cite{gordaliza2019obtaining, laclau2021all}. 
A more recent and principled direction involves disentangled representation learning. This notion requires that a model's prediction for a node remain unchanged in a hypothetical scenario where the sensitive attribute label differs, while all other factors remain constant. Methods such as NIFTY \cite{agarwal2021towards} and GEAR \cite{ma2022learning} generate counterfactual graph variants by perturbating node features or structure, and align representations between factual and counterfactual versions. However, these approaches often rely on heuristics that may produce unrealistic counterfactuals, potentially harming model performance.
CAF (Counterfactual Augmented Fair GNN) \cite{guo2023towards} addresses these limitations by explicitly modeling the data-generating process using a Structural Causal Model. CAF disentangles node representations into content and environmental representations: the former is label-relevant and invariant to sensitive attributes, while the latter captures variation due to those attributes. This disentangled design leads to improved fairness. FairGB \cite{li2024rethinking} mitigates unfairness in GNNs through group balancing, using real counterfactual node mixup and a contribution alignment loss. FairSAD \cite{zhu2024fair} enhances the fair separation of sensitive attribute-related information into an independent component to mitigate its impact by using a channel masking mechanism to adaptively identify the component related to sensitive attributes. Recently, FairSR \cite{liu5320563rethinking} integrates structural rebalancing with adversarial learning to ensure equitable outcomes for nodes with varying degrees but similar functionalities in prediction tasks.

This work differs from existing methods in three key aspects. (i) Unlike FairDrop, which is an in-processing method that reduces homophily with respect to sensitive attribute labels, our approach simultaneously increases homophily with respect to class labels to preserve class label information, requires no hyper-parameters, and operates as a pre-processing step. (ii) We improve the CAF framework by introducing two novel training objectives that promote disentanglement and fairness more effectively than the original objective.

\section{Problem Statement and Background Knowledge}

\subsection{Problem statement}
\paragraph{Notations.} In this study, our training data consists of $n$ individuals, each described by a feature vector $ x_i \in \mathbb{R}^d$. The task is to predict a binary label $Y \in \{0, 1\}$, which is observed only for a subset of individuals. We consider a semi-supervised setting where the index set $\{1, \dots, n\}$ is partitioned into a labeled subset \( \mathcal{L} = \{1, \dots, \ell\} \) and an unlabeled subset \( \mathcal{U} = \{\ell+1, \dots, \ell + u\} \), such that \( n = \ell + u \). The binary class label \( Y_i \in \{0, 1\} \) is observed only for individuals in \( \mathcal{L}\). Each individual is also associated with a binary sensitive attribute $S \in \{0, 1\}$. At training, we have access to the sensitive attribute labels for all individuals like CAF. 

\paragraph{From tabular data to graphs.} We model the relationships between individuals using a graph, where each node represents an observed individual. Formally, a graph is defined as a pair \( G = (\mathcal{V}, \mathcal{E}) \), where \( \mathcal{V} \) is the set of nodes with \( |\mathcal{V}| = n \), and \( \mathcal{E} \subseteq \mathcal{V} \times \mathcal{V} \) is the set of edges. The connectivity between nodes is encoded by the adjacency matrix \( \mathcal{A} \in \mathbb{R}^{n \times n} \), where \( \mathcal{A}_{i,j} = 1 \) if there is an edge from node \( i \) to node \( j \), and \( \mathcal{A}_{i,j} = 0 \) otherwise. For standard definitions of graphs, we refer the reader to \cite{bondy2008graph}. Finally, the feature vector $x_i$ can also be used to construct the graph. In this paper, we focus on the case of a binary sensitive attribute. However, following the approaches of \cite{ma2022learning, guo2023towards}, our model naturally extends to accommodate more general settings, including weighted graphs and non-binary sensitive attributes.
\paragraph{Objective}
Our objective is to learn a \emph{fair classifier} from graph-structured data: that is, a model that predicts the label \( Y_i \in \{0,1\} \) for each node \( i \), while minimizing the influence of the sensitive attribute \( s_i \in \{0,1\} \) on the predictions. To achieve this, we follow a two-step approach. First, we learn a node representation \( h_i \in \mathbb{R}^{d'} \) that captures relevant information from both the node features and the graph structure via an encoder \( f_{\theta}: \mathcal{X}^{n \times d} \times \mathcal{A}^{n \times n} \to \mathcal{H}^{n \times d'} \), typically implemented as a Graph Neural Network (GNN), such as GCN \cite{kipf2016semi} and GraphSAGE \cite{hamilton2017inductive}. This encoder aggregates information from each node’s neighborhood, learning low-dimensional embeddings that preserve both attribute and structural information. Second, we apply a classifier \( g_{\phi}: \mathbb{R}^{d'} \to [0,1] \) to the learned embeddings to produce predictions \( \hat{Y}_i = g_{\phi}(h_i) \).
Importantly, we seek embeddings that are not only predictive, but also \emph{fair}, in the sense that they are invariant or at least minimally informative with respect to the sensitive attribute \( S \). This separation between utility and bias is central to our approach to fairness-aware learning.

\subsection{Fairness In Graph Neural Networks}
Researchers assess fairness for Graph Neural Networks across various fundamental levels. In their work \cite{chen2023fairness}, the authors classify fairness evaluation metrics for Graph Neural Networks into prediction-level metrics, graph-level metrics, neighborhood-level metrics, and embedding-level metrics. In our work, we focus on prediction-level and graph-level metrics, which we explain in this section.
\subsubsection{Prediction-level (Statistical Parity (SP) \cite{dai2021say})} The statistical parity (SP) metric, also known as demographic parity (DP), requires that the predictions made by the model be independent of the sensitive attribute label ($s$). It is defined as follows:
\begin{equation}
    \Delta_{SP} = {\lvert\mathcal{P}({\hat{y} = 1 \lvert \ s = 0)}) - \mathcal{P}({\hat{y} = 1 \lvert \ s = 1)}\lvert}
\end{equation}
Where \( \Delta_{SP} = 0 \) implies that all groups have the same selection rates, indicating complete fairness. Statistical parity measures the gap in preferential treatment between different groups. However, it's important to note that \( \Delta_{SP} \) does not consider whether individuals are qualified or not, as it does not take into account the ground-truth label ($y$).

\subsubsection{Prediction-level (Equal Opportunity (EO) \cite{dai2021say})} This principle dictates that the probability of a positive outcome being predicted for instances in a positive class should be the same for all subgroup members. It is defined as follows:
\begin{equation}    
    \Delta_{EO} = \left\lvert \mathcal{P}(\hat{y} = 1 \,|\, {y} = 1 , s = 0) \right. -
    \left. \mathcal{P}(\hat{y} = 1 \,|\, {y} = 1 , s = 1) \right\rvert
\end{equation}
Here \( \Delta_{EO} = 0 \) implies equal true positive rates across subgroups, indicating fairness.

\subsubsection{Graph-level Homophily Ratio \cite{chen2023fairness}} Graph-level fairness metrics consider the biases that arise from both the graph topology G and sensitive attributes. These metrics are largely based on the notion of homophily. The Homophily global ratio denotes the number of nodes which are connected sharing the same sensitive attribute label. It depends both on the sensitive attribute labels and on the geometry of the graph. It is defined as follows:
\begin{equation}
hr_s = \frac{\left| \{ (u, v) \in \mathcal{E} \mid s_u = s_v \} \right|}{|\mathcal{E}|}.
\end{equation}

In a similar way, we can define the global homophily (edge-homophily) ratio \cite{zhu2020beyond} with respect to class labels as follows.  
\begin{equation}
hr_c = \frac{\left| \{ (u, v) \in \mathcal{E} \mid y_u = y_v \} \right|}{|\mathcal{E}|}.
\end{equation}

These two measures quantify the level of connectivity between nodes with similar sensitive attribute labels and between nodes with similar class labels. 

\subsection{The Counterfactual Augmented Fair Graph Neural Network (CAF)}
We recall the CAF method of Guo et al.~\cite{guo2023towards}, which forms the basis of our algorithm. The framework partitions the latent representation $\mathcal{H}$ into two components: a content representation $C$ (label-relevant and insensitive to sensitive attributes) and an environmental representation $E$ (primarily dependent on sensitive attributes). Formally,
\begin{equation}
    f_{\theta}(\mathcal{X}, \mathcal{A}) = \mathcal{H} = [ {C} , {E} ]
\end{equation}

The first $d_{c}$ columns of $\mathcal{H}$ constitute the content matrix $C$, used exclusively for prediction, while the next $d_{e}$ columns form the environmental matrix $E$. For balance, $d_c$ and $d_e$ are set equal.
Thus, the latent representation of each node $i$ can be written as:
\begin{equation}
    h_i = [c_i, e_i].
\end{equation}
Note that computing counterfactuals requires the estimated labels of all nodes. In the pre-training phase, CAF employs available labels to guide representation learning. A pre-training encoder $f_{\theta}$ (GraphSAGE) and a classifier $f_{\phi}$ (Sigmoid classifier), parameterized by $\phi$, are jointly trained on the content representations to predict the class distribution of labeled nodes, optimized with cross-entropy loss:
\begin{equation}\label{eq:forecast}
f_{\phi}(c_i) = \hat{y}_i
\end{equation}
The trained model is then applied to the unlabeled nodes to obtain pseudo-labels $\hat{y}$, which are subsequently used for counterfactual construction. For each node $i$, two types of counterfactuals are defined:
\begin{itemize}
    \item nodes with the same pseudo-label but a different sensitive attribute label, i.e.,  $\hat{y}_i = \hat{y}_j$ and $s_i \neq s_j$;
    \item nodes with the different pseudo-label but same sensitive attribute label, i.e., $\hat{y}_i \neq \hat{y}_j$ and $s_i = s_j$.
\end{itemize}

Since selecting counterfactuals directly in the original data space is computationally expensive due to graph distance calculations, CAF instead measures distances in the latent space $\mathcal{H}$, thus reducing computational cost. 

Formally, let $h_{i}^{e}$ denote the counterfactual embedding of node $i$ that shares a similar pseudo-label but has a different sensitive attribute label, and let $h_{i}^{c}$ denote the counterfactual embedding of node $i$ that has a dissimilar pseudo-label but the same sensitive attribute label, defined as:
\begin{equation}\label{eq:cf_e}
h_{i}^{e} \in \arg\min_{h_j \in \mathcal{H}} \bigl\{ \|h_i - h_j\|_2^2 \;\big|\; \hat{y}_i = \hat{y}_j, \, s_i \neq s_j \bigr\},
\end{equation}
\begin{equation}\label{eq:cf_c}
h_{i}^{c} \in \arg\min_{h_j \in \mathcal{H}} \bigl\{ \|h_i - h_j\|_2^2 \;\big|\; \hat{y}_i \neq \hat{y}_j, \, s_i = s_j \bigr\},
\end{equation}
with $h_{i}^{e}=(c_i^e,e_i^e)$ and $h_{i}^{c}=(c_i^c,e_i^c)$.

Note that for each factual input, multiple counterfactuals can be obtained by selecting a set $\mathcal{K}$ of counterfactuals rather than a single one, as indicated in Equations.~\ref{eq:cf_e} and \ref{eq:cf_c}.

\paragraph{Training loss of CAF}
The learning process is driven by three loss terms. 
\begin{enumerate}
    \item A prediction loss $\mathcal{L}_\mathrm{pred}$, which is a standard cross-entropy loss for label prediction based on the content representation. 
    
    \item An invariance loss $\mathcal{L}_\mathrm{inv}$ that encourages the content and environmental components to remain independent of the sensitive attribute label, using the counterfactuals. The invariance loss is defined as:
    \begin{eqnarray}
          \mathcal{L}_\mathrm{inv} = \frac{1}{|\mathcal{V}|\cdot K} & &  \sum_{v_i \in V} \sum_{k=1}^{K} [ \text{dis}(c_i, c_{i}^{e_{k}}) 
+  \text{dis}(e_i, e_{i}^{c_{k}}) 
\nonumber\\
&+&  \gamma K\cdot \left| \cos(c_i, e_i) \right| ]
    \end{eqnarray}
where $dis(.,.)$ is a distance metric (e.g., cosine or $L2$ distance), and $\gamma$ controls the orthogonality between $c_i$ and $e_i$.
\item A sufficiency loss $\mathcal{L}_\mathrm{suf}$ that ensures that the learned representation preserves graph structure by maintaining the same topological properties as the original input graph:
\begin{eqnarray}
      \mathcal{L}_{\mathrm{suf}} = \frac{1}{|\mathcal{A}| + |\mathcal{A}^-|}
    &\times& \sum_{({i}, {j})\in \mathcal{A} \cup \mathcal{A}^-} -a_{i,j} \log p_{i,j} \nonumber \\ &-& (1 - a_{i,j}) \log (1 - p_{i,j}),
\end{eqnarray}

where $\mathcal{A}^{-}$ is the set of negative edges, and $p_{i,j}$ is the predicted probability of a link between nodes $i$ and $j$.
\end{enumerate}
The total loss function is given by 
\begin{equation}
\min_{\theta, \phi} \mathcal{L} = \mathcal{L}_\mathrm{pred} + \alpha \mathcal{L}_\mathrm{inv} + \beta \mathcal{L}_\mathrm{suf}
\end{equation}
where $\theta$ and $\phi$ are the GNN encoder and prediction head parameters, $\alpha$ and $\beta$ are hyperparameters controlling the trade-off between fairness and accuracy. 

\paragraph{Limitations of CAF}  
CAF improves fairness by disentangling node representations into content ($C$) and environmental ($E$) representations using two types of counterfactual equations~\ref{eq:cf_e} and \ref{eq:cf_c}.  

$C$ is encouraged to align nodes with the same class pseudo-label but different sensitive attribute labels, while $E$ is encouraged to align nodes with the same sensitive attribute label but different class pseudo-labels.  
This design aims to prevent class-relevant information from leaking into $E$ and sensitive-relevant information from leaking into $C$, with orthogonality constraints further supporting this separation.  
However, two key limitations remain.  

\begin{itemize}
\item A limitation of CAF is that nodes sharing the same class pseudo-label and the same sensitive attribute label are excluded from the invariance loss, and thus this information is not considered in $C$. Although prediction loss $\mathcal{L}_\mathrm{pred}$ captures this information, it is not a sufficiently strong training objective to ensure that nodes with similar class labels learn similar embeddings, while nodes with dissimilar class labels are pushed apart in the embedding space. Consequently, the model cannot fully exploit all the class-relevant information when learning node representations.

\item Since $E$ is only trained on intra-group pairs (same sensitive attribute label, different class pseudo-labels), inter-group variation is not modeled, which limits separation across sensitive groups.  
\end{itemize}

As a result, $C$ and $E$ may not fully capture class and sensitive information, which can compromise both fairness and predictive performance.

\paragraph{Our Contribution.}
To address these limitations, we propose a novel algorithm that not only mitigates incomplete disentanglement 
but also reduces \emph{topology bias} in graphs. 
The details of our proposed method are described in the following section.

\section{The HSCCAF Framework}
We propose the \textit{Homophily-Aware Supervised Contrastive Counterfactual Augmented Fair Graph Neural Network} (HSCCAF), which extends the Counterfactual Augmented Fairness (CAF) model \cite{guo2023towards}. Specifically, HSCCAF introduces fairness-aware graph editing and incorporates two additional training objectives: \textit{Supervised Contrastive Loss} and \textit{Environmental Loss}. These components are designed to improve fairness while preserving predictive accuracy. Figure~\ref{fig:model} illustrates the overall architecture.
\begin{figure*}[h]
    \centering
\includegraphics[width=0.9\textwidth]{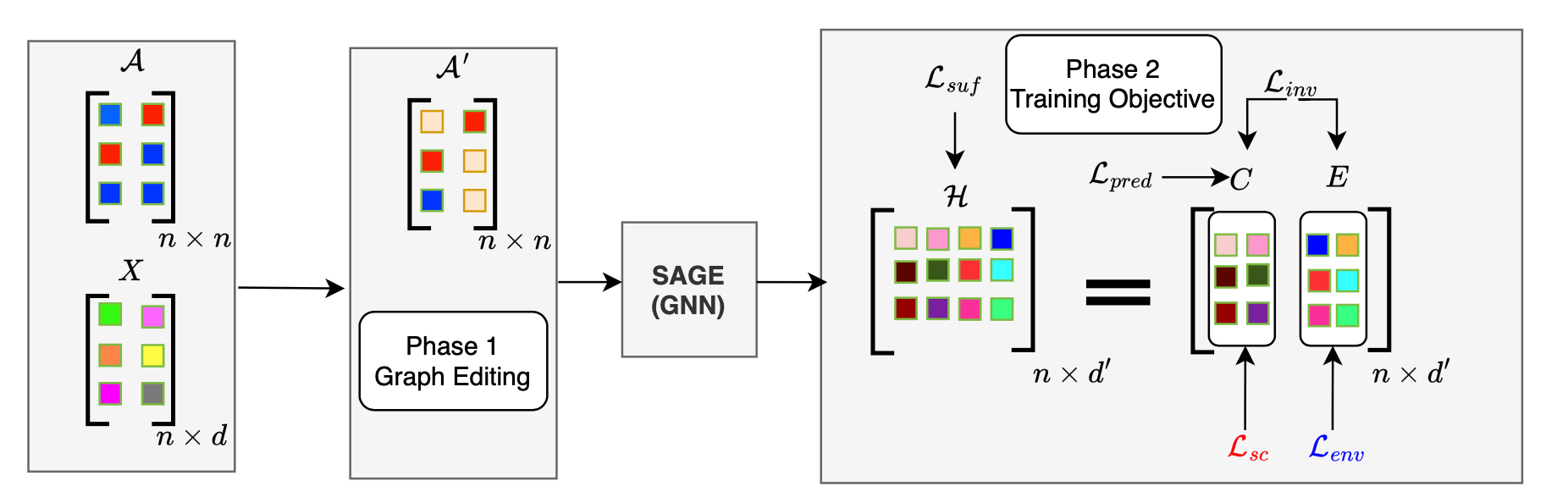}
    \caption{Overview of HSCCAF. The framework extends CAF with (i) a fairness-aware graph editing phase that edits the input graph to adjust homophily, and (ii) two additional loss terms that explicitly regularize the content and environmental representations.}
    \label{fig:model}
\end{figure*}

\subsection{Fairness-Aware Graph Editing}
The objective of this phase is to adjust the topology of the input graph in a way that balances predictive accuracy and fairness. To this end, we selectively edit edges with two complementary goals. First, we aim to \textbf{increase homophily with respect to class labels} ($hr_c$), which strengthens the label-consistent neighborhoods that most GNNs rely on for effective message passing, thereby preserving class label information. Second, we aim to \textbf{reduce homophily with respect to sensitive attribute labels} ($hr_s$), which prevents the model from overemphasizing demographic similarities, promotes fairness by mitigating bias. This fairness-aware edge editing acts as a pre-processing step that reshapes the graph before message passing.

We adopt a simple, interpretable edge-editing strategy.
\begin{itemize}
    \item Increase $hr_c$ to preserve class label information;
    \item Decrease $hr_s$ to mitigate topology-induced bias.
\end{itemize}

To operationalize this, edges are categorized into four types using pseudo-labels $\hat{Y}$ (estimated class labels) and sensitive attribute labels $S$:
\vspace{0.5em}

\begin{tabular}{@{}l l l@{}}
\textbf{Type I:}   & $\hat{y}_i = \hat{y}_j,$ & $s_i = s_j$ \\
\textbf{Type II:}  & $\hat{y}_i = \hat{y}_j,$ & $s_i \neq s_j$ \\
\textbf{Type III:} & $\hat{y}_i \neq \hat{y}_j,$ & $s_i = s_j$ \\
\textbf{Type IV:}  & $\hat{y}_i \neq \hat{y}_j,$ & $s_i \neq s_j$
\end{tabular}

\vspace{0.5em}
Among these, \textbf{Type~III} edges are the most harmful, as they simultaneously reduce class-label consistency and reinforce sensitive-group homogeneity. Therefore, we remove all of these edges during graph editing. In Appendix \ref{appendix}, we provide a proof that removing all Type~III edges is better than removing only a subset of them.

This phase introduces no additional hyperparameters and runs in linear time, \(\mathcal{O}(|\mathcal{E}|)\), requiring only a single pass through the edge list with access to pseudo-labels and sensitive attribute labels. A formal proof of the monotonic effect on $hr_c$ and $hr_s$ is provided in Appendix~\ref{appendix}.

\subsection{Supervised Contrastive Loss}
Supervised contrastive loss, $\mathcal{L}_\mathrm{sc}$, plays a crucial role in Supervised Contrastive Learning~\cite{khosla2020supervised}, encourages content representations of nodes of the same class to be similar while pushing apart nodes from different classes. Unlike CAF, which relies on counterfactual pseudo-labels to enforce invariance, our approach directly enforces intra-class similarity using only the labeled nodes. This helps the model learn a robust content ($C$) representation, improving fairness by reducing the implicit encoding of sensitive attributes, without compromising accuracy.

Intuitively, the loss pulls together nodes that share the same class label in the latent space while pushing apart nodes from different classes, ensuring label-consistent embeddings independent of sensitive attribute labels. This directly strengthens the quality of the content representation for prediction.

Formally, supervised contrastive loss is defined as:

\begin{equation}
\mathcal{L}_\mathrm{sc}
= - \sum_{i=1}^{l} \frac{1}{\lvert \mathcal{P}(i)\rvert}
\sum_{p\in\mathcal{P}(i)}
\log\left(
\frac{\exp\!\big(\Phi(\mathbf{c}_i,\mathbf{c}_p)\big)}
{\displaystyle\sum_{a\in\mathcal{V}(i)} \exp\!\big(\Phi(\mathbf{c}_i,\mathbf{c}_a)\big)}
\right)
\end{equation}

\begin{equation}
\Phi(\mathbf{c}_i,\mathbf{c}_j)
= \frac{1 + \operatorname{cos}(\mathbf{c}_i,\mathbf{c}_j)}
{1 + k\big(1 - \operatorname{cos}(\mathbf{c}_i,\mathbf{c}_j)\big)} - 1,
\end{equation}

where $\mathcal{P}(i) \equiv \{ p \in \mathcal{V}(i): \hat{\mathbf{y}}_{p} \equiv \hat{\mathbf{y}}_{i} \}$ is the set of positive samples that share the same class label as node $i$, and $\mathcal{V}(i)$ denotes all other nodes considered for contrast.

Instead of the standard dot product, we use the t-vMF (Student-t von Mises–Fisher) similarity~\cite{kobayashi2021t}, which better captures the angular similarity on the hypersphere and provides a more discriminative metric, improving generalization performance.

\subsection{Environmental loss}
Environmental loss ensures that samples from different sensitive groups have dissimilar environmental representations. This separation is maintained regardless of the class labels. This is important because CAF encourages samples that share the same sensitive group but differ in class labels to have similar embeddings. In contrast, our goal is to enforce dissimilar embeddings for samples with different sensitive groups, regardless of their class labels. By introducing this loss, we inject more information about the sensitive attribute into the environmental representation, helping the model learn a more robust environmental (E) representation. As a result, this improves both accuracy and fairness.

The Environmental Loss, $\mathcal{L}_{\mathrm{env}}$, operates by encouraging nodes with dissimilar sensitive attribute labels to have dissimilar environmental representations in the latent space. $\mathcal{L}_{\mathrm{env}}$ is defined as follows:

\begin{equation}
    \mathcal{L}_{\mathrm{env}} = - \frac{1}{n} \sum_{i=1}^n \frac{1}{K'} \sum_{j=1}^{K'} \text{dis}(e_i, e_j)
\end{equation}

Here, $K'$ denotes the number of nodes with dissimilar sensitive attribute labels, defined as a hyperparameter, whose environmental representations are close to the current node. The distance measure $\mathrm{dis}$ is defined as the L2 distance between the environmental representations of samples $e_i$ and $e_j$ from different sensitive groups:
\begin{equation}
\mathrm{dis}(e_i, e_j) = \| e_i - e_j \|_2 .
\end{equation}
\subsection{Putting Everything Together}
The training of HSCCAF proceeds in three stages, summarized in Algorithm~\ref{alg:HSSCAF}:

\paragraph{Pre-training} Following CAF~\cite{guo2023towards}, we employ a GNN encoder (GraphSAGE) together with a Sigmoid classifier $f_{\phi}$. The model is pre-trained on labeled nodes using cross-entropy loss to learn content representations and predict class distributions. The trained predictor is then applied to unlabeled nodes to infer pseudo-labels, which are used in the subsequent stages of our proposed method:

\begin{equation}
\min_{\theta, \phi} \mathcal{L} = \mathcal{L}_\mathrm{pred}.
\end{equation}

\paragraph{Phase 1: Fairness-Aware Graph Editing} Using pseudo-labels and sensitive attribute labels, harmful edges (Type~III) are removed to increase homophily with respect to class labels ($hr_c$) while decreasing homophily with respect to sensitive attribute labels ($hr_s$). This pre-processing step edits the graph to promote both accuracy and fairness.

\paragraph{Phase 2: Full Training} We train the model on the edited graph using the following loss:
\begin{equation}
\min_{\theta, \phi} \mathcal{L} = \mathcal{L}_\mathrm{pred} + \alpha \mathcal{L}_\mathrm{inv} + \beta \mathcal{L}_\mathrm{suf} + \omega \mathcal{L}_\mathrm{sc} + \eta \mathcal{L}_\mathrm{env},
\end{equation} 

where $\alpha$, $\beta$, $\omega$, and $\eta$ are trade-off parameters control the relative influence of each term (see Section~\ref{sec:exp}).  

The latent representation is partitioned into content ($C$) and environmental ($E$) components. Performance and fairness improve when sensitive attribute information is confined to $E$, leaving $C$ focused on label-relevant features. To achieve this, Supervised Contrastive Loss and Environmental Loss are employed to regularize $C$ and $E$, respectively, ensuring that each part appropriately incorporates the intended information.

The Algorithm~\ref{alg:HSSCAF} summarizes all three phases of HSCCAF, including pre-training, graph editing, and full training.

\begin{algorithm}[H]
\caption{HSCCAF}
\label{alg:HSSCAF}
\begin{algorithmic}[1]
\REQUIRE Graph $G=(V,A)$, features $X$, sensitive attributes $s$, labels $y$ (partially known)
\ENSURE Trained GraphSAGE encoder parameters $\theta$ and predictor parameters $\phi$
\STATE \textbf{Initialize Encoder (2-layers GraphSAGE):} $f_{\theta}(X,A) \to H \in \mathbb{R}^{n \times (d_c+d_e)}$ and $d_c = d_e$
\STATE \textbf{Initialize Predictor (1-layer Sigmoid Classifier):} $f_{\phi}(C)\to \hat{y}$
\STATE \textbf{Pre-training Phase}
  \FOR{$t=1$ \textbf{to} $T_{\text{pre}}$}
    \STATE $H \gets f_{\theta}(X,A)$; partition into $C,E$
    \STATE $\hat{y} \gets f_{\phi}(C)$
    \STATE Update $\theta,\phi$ by minimizing $\mathcal{L}_\mathrm{pred}$
  \ENDFOR
\STATE \textbf{Phase1 (Fairness-Aware Graph Editing)}
\FOR{each edge $(i, j) \in \mathcal{E}$}
    \IF{$\hat{y}_i \ne \hat{y}_j$ \AND $s_i = s_j$}
        \STATE Remove edge $(i, j)$ from $\mathcal{E}$
    \ENDIF
\ENDFOR
\STATE \textbf{Phase2}
  \FOR{$t=1$ \textbf{to} $T_{\text{train}}$}
    \STATE $H \gets f_{\theta}(X,A)$; partition into $C,E$
    \IF{$t \bmod 5 = 0$}
      \STATE $\hat{y} \gets f_{\phi}(C)$
      \FOR{each node $i \in V$}
        \STATE Select $K$ nearest $h_i^{e_k}$ s.t.\ $(\hat{y}_i=\hat{y}_j,\; s_i \neq s_j)$
        \STATE Select $K$ nearest $h_i^{c_k}$ s.t.\ $(\hat{y}_i\neq\hat{y}_j,\; s_i=s_j)$
      \ENDFOR
    \ENDIF
    \STATE Compute $\mathcal{L} = \mathcal{L}_\mathrm{pred} + \alpha \mathcal{L}_\mathrm{inv} + \beta \mathcal{L}_\mathrm{suf} + \omega \mathcal{L}_\mathrm{sc} + \eta \mathcal{L}_\mathrm{env}$
    \STATE Update $\theta, \phi$ by minimizing $\mathcal{L}$
  \ENDFOR
\end{algorithmic}
\end{algorithm}

\section{Experiments}
\label{sec:exp}
In this section, we assess the proposed model's performance on diverse tabular real-world datasets, encompassing German \cite{asuncion2007uci}, Bail \cite{jordan2015effect}, Credit Defaulter \cite{yeh2009comparisons}, NBA \cite{dai2021say} and Pokec-n \cite{takac2012data}.

\subsection{Datasets}
We used the following datasets.
\begin{enumerate}
    \item \textbf{German \cite{asuncion2007uci}:} 
          The dataset consists of 1000 individual clients from a German bank, and these individuals, acting as nodes, are linked when there is a significant similarity in their credit accounts. The main objective is to evaluate the credit risk level for each individual, emphasizing the sensitive attribute of ``gender''.
    \item \textbf{Bail \cite{jordan2015effect}:} 
          The dataset comprises 18,876 individuals granted bail by US state courts between 1990 and 2009. These individuals, serving as nodes, are linked based on similarities in their demographics and prior criminal histories. The primary goal is to categorize defendants as either on bail or not, with a focus on the sensitive attribute ``race''.
    \item \textbf{Credit Defaulter \cite{yeh2009comparisons}:}
    The dataset includes 30,000 individual credit card holders, where these individuals, acting as nodes, are linked based on high similarity in payment information. The main objective is to predict whether an individual will default on credit card payments, with particular attention to the sensitive attribute of ``age''.
    \item \textbf{NBA \cite{dai2021say}:}
    The dataset includes 400 NBA players with statistics from the 2016–2017 season, along with information such as nationality, age and salary. The players act as nodes in a graph, linked based on their Twitter connections. The main task is to predict whether a player’s salary is above or below the league median, with ``nationality'' (U.S. or non-U.S.) used as a sensitive attribute.
    \item \textbf{Pokec-n \cite{takac2012data}:}  
    The dataset includes 66,569 individual users from a popular online social network in Slovakia, represented as nodes in a graph where edges capture friendship relations between them. The main task is to predict each user’s working field, with particular attention to the sensitive attribute ``region''.
\end{enumerate}

\subsection{Experimental setup}
We compare the proposed framework with state-of-the-art node classification methods, including GCN~\cite{kipf2016semi}, GraphSAGE~\cite{hamilton2017inductive}, and GIN~\cite{xu2019powerful}. We also evaluate its performance against fair node classification approaches, such as FairGNN~\cite{dai2021say}, FairVGNN~\cite{dong2022edits}, FairSAD \cite{zhu2024fair}, and the counterfactual fairness method CAF~\cite{guo2023towards}\footnote{\url{https://github.com/TimeLovercc/CAF-GNN}} and FairGB~\cite{li2024rethinking}.

\paragraph{Splitting strategy} For all datasets, we follow the train/validation/test split protocol from~\cite{guo2023towards}. To ensure robust evaluation, we report average results over five random splits. All methods are tested on the same splits to enable fair and consistent comparisons.

\paragraph{Hyper-parameter tuning}
Each method is tuned independently using a grid search over its respective hyper-parameter space. The configuration yielding the best validation performance is selected. This tuning protocol is applied consistently across all methods. Additionally, we fix the learning rate to 0.01 and use a constant random seed for reproducibility.

For our proposed method, the hyper-parameters are chosen as follows: $K$ and $K'$ from \{2, 5, 10\}; $\alpha$ from \{0.2, 0.5, 0.9, 5, 10\}; $\beta$ from \{0.5, 1\}; $\gamma$ from \{0.02, 0.1, 1\}; $\omega$ from \{0.03, 0.09, 0.3, 0.7, 1\}; and $\eta$ from \{0.06, 0.07, 0.08, 0.09, 0.1, 0.3, 0.8\}.

To select the best epoch during training, we compute the following validation score:
\begin{equation}
\text{Score} = \text{BACC} + \frac{1}{2} \left[ (100 - \Delta_{EO}) + (100 - \Delta_{SP}) \right]
\end{equation}
This score balances classification accuracy and fairness, with the epoch achieving the highest score selected. Training is capped at 100 epochs.

\subsection{Method comparison}
We assess the model's performance (HSCCAF). Tables \ref{tab:German}, \ref{tab:Bail}, \ref{tab:Credit} and  \ref{tab:NBA} present the Balanced Accuracy (BACC), AUC and F1 score as percentages, along with their standard deviations across five random splits, providing insights into node classification proficiency. In terms of fairness evaluation, guided by \cite{10.1145/3437963.3441752}, we employ two widely recognized group fairness metrics specifically, statistical parity ($\Delta_{SP}$) and equality opportunity ($\Delta_{EO}$). These metrics are also reported as percentages. Furthermore, these tables include a relative comparison between our proposed method (HSCCAF) and the baseline methods, where a positive value indicates that HSCCAF outperforms the baseline, and a negative value indicates lower performance.

\subsection{Effect of the proposed HSCCAF for fair classification task}
We evaluate HSCCAF against vanilla GNN baselines (GCN, SAGE, GIN), fairness-aware methods (FairGNN, FairVGNN, CAF FairGB and FairSAD), and our foundational method (HSCCAF) on five benchmark datasets. The results, reported in Tables~\ref{tab:German}--\ref{tab:pocke}, reflect a consistent trade-off between predictive performance and fairness metrics. 

On the German dataset (Table~\ref{tab:German}), HSCCAF(SAGE) achieves the best fairness performance among all methods, with $\Delta_{SP}=2.86$ and $\Delta_{EO}=3.39$. This represents a significant improvement over both vanilla GNNs and other fairness-aware approaches. In terms of classification performance, HSCCAF is competitive, achieving a balanced accuracy of 60.02\% and an AUC of 63.74\%, close to FairGB. Notably, while FairGB performs slightly better in terms of AUC and BACC, it does so at the cost of a much larger fairness gap. This result illustrates HSCCAF’s ability to mitigate unfairness while preserving predictive performance.

\begin{table}[h]
    \centering
    \caption{Evaluation of Node Classification and Fairness Performance on the German Dataset.}\label{tab:German}
    \resizebox{0.80\textwidth}{!}{%
    \begin{tabular}{lccccc}
        \toprule
        \textbf{Method} & \textbf{BACC($\uparrow$)} & \textbf{AUC($\uparrow$)} & \textbf{F1($\uparrow$)} & \boldmath$\Delta_{SP}(\downarrow)$ & \boldmath$\Delta_{EO}(\downarrow)$ \\
        \midrule
        \textbf{GCN} & 57.71(6.14) & 60.78(8.35) & \textbf{73.25(6.90)} & \uline{3.71(1.87)} & 5.76(3.17)\\
        
        \textbf{SAGE} & \uline{60.93(6.25)} & \uline{65.84(4.58)} & 68.51(10.19) & 7.08(4.87) & 6.31(5.46) \\
        
        \textbf{GIN} & 57.52(2.78) & 59.58(4.19) & 59.10(16.16) & 9.59(5.98) & 10.14(4.92) \\
        
        \textbf{FairGNN(GCN)} & 59.71(5.46) & \uline{64.73(3.75)} & 65.57(5.79) & 9.39(3.08) & 7.97(4.58) \\
        
        \textbf{FairVGNN(GCN)} &  \uline{59.79(4.37)} & 63.97(3.43) & 66.45(6.49) & 8.98(4.28) & 10.93(4.44) \\
        
        \textbf{CAF(SAGE)} &  57.58(5.36) & 61.87(8.13) & 62.96(9.18) & 4.46(3.06) & 6.15(0.85)  \\
        
        \textbf{FairGB(SAGE)} &  \textbf{61.49(3.44)} & \textbf{66.48(4.08)} & \uline{68.81(7.14)} & 7.46(7.00) & 7.69 (5.88)  \\
        
        \textbf{FairSAD(GCN)} & 57.66(3.98) & 61.15(3.14) & 63.67(9.26) & 4.61(2.78) & \uline{5.39(2.70)}  \\
        
        \textbf{HSCCAF(SAGE)} & 60.02(2.93) & 63.74(3.87) & 65.50(4.70) & \textbf{2.86(2.20)} & \textbf{3.39(1.05)} \\
        \midrule
        \textbf{Relative $\Delta\%$ (HSCCAF vs. GCN)} & +4.01\% & +4.87\% & -10.60\% & +22.91\% & +41.14\% \\
        
        \textbf{Relative $\Delta\%$ (HSCCAF vs. SAGE)} & -1.49\% & -3.18\% & -4.39\% & +59.60\% & +46.27\% \\
        
        \textbf{Relative $\Delta\%$ (HSCCAF vs. GIN)} & +4.35\% & +6.98\% & +10.83\% & +70.17\% & +66.56\% \\
        
        \textbf{Relative $\Delta\%$ (HSCCAF vs. FairGNN)} & +0.52\% & -1.53\% & -0.11\% & +69.55\% & +57.46\% \\
        
        \textbf{Relative $\Delta\%$ (HSCCAF vs. FairVGNN)} & +0.38\% & -0.36\% & -1.43\% & +68.15\% & +69.00\% \\
        
        \textbf{Relative $\Delta\%$ (HSCCAF vs. CAF)} & +4.24\% & +3.02\% & +4.03\% & +35.87\% & +44.88\% \\
        
        \textbf{Relative $\Delta\%$ (HSCCAF vs. FairGB)} & -2.39\% & -4.12\% & -4.81\% & +61.65\% & +55.90\% \\
        
        \textbf{Relative $\Delta\%$ (HSCCAF vs. FairSAD)} & +4.09\% & +4.23\% & +2.87\% & +37.96\% & +37.10\% \\
        \bottomrule
    \end{tabular}
    }
\end{table}

On the Bail dataset (Table~\ref{tab:Bail}), HSCCAF(SAGE) achieves the strongest fairness performance, with $\Delta_{SP}=5.64$ and $\Delta_{EO}=1.55$, substantially reducing disparities compared to other methods. Although CAF(SAGE) attains the best classification results, with the highest balanced accuracy (85.58\%), AUC (91.84\%), and F1-score (82.16), it does so at the cost of larger fairness gaps. In contrast, HSCCAF strikes a favorable balance by maintaining competitive predictive performance (BACC = 84.67\%, AUC = 90.49\%, F1 = 81.03\%) while significantly improving fairness. These findings highlight HSCCAF’s ability to mitigate unfairness without severely sacrificing accuracy, offering a more balanced trade-off than both fairness-aware baselines and standard GNN models.

\begin{table}[h]
    \centering
    \caption{Evaluation of Node Classification and Fairness Performance on the Bail Dataset.}\label{tab:Bail}
    \resizebox{0.80\textwidth}{!}{%
    \begin{tabular}{lccccc}
        \toprule
        \textbf{Method} & \textbf{BACC($\uparrow$)} & \textbf{AUC($\uparrow$)} & \textbf{F1($\uparrow$)} & \boldmath$\Delta_{SP}(\downarrow)$ & \boldmath$\Delta_{EO}(\downarrow)$ \\
        \midrule
        \textbf{GCN} & 83.46(0.81) & 89.19(0.97) & 79.60(1.14) & 9.10(2.77) & 5.45(2.98)\\
        
        \textbf{SAGE} & 81.76(2.63) & 88.46(2.19) & 77.26(3.35) & 6.68(5.26) & 5.51(3.34) \\
        
        \textbf{GIN} & 78.80(1.33) & 85.57(1.39) & 73.51(1.68) & 9.35(4.85) & 5.93(4.85) \\
        
        \textbf{FairGNN(GCN)} & 78.31(1.78) & 80.22(4.99) & 67.80(4.14) & \textbf{3.30(2.24)} & 3.62(1.35) \\
        
        \textbf{FairVGNN(GCN)} & 78.81(2.45) & 85.79(1.76) & 73.76(2.93) & 4.86(1.48) & 2.72(1.33) \\
        
        \textbf{CAF(SAGE)} & \textbf{85.58(0.77)} & \textbf{91.84(0.87)} & \textbf{82.16(1.02)} & 6.39(2.52) & \uline{2.06(1.81)} \\
        
         \textbf{FairGB(SAGE)} & 81.90(2.28) & 88.88(1.41) & 77.56(2.77) & \uline{4.16(5.87)} & 3.60(4.17) \\
        
        \textbf{FairSAD(GCN)} & 79.92(5.13) & 85.98(4.98) & 75.15(6.64) & 4.52(2.99) & 3.05(0.50) \\
        
        \textbf{HSCCAF(SAGE)} & \uline{84.67(1.50)} & \uline{90.49(0.90)} & \uline{81.03(1.99)} & 5.64(1.89) &  \textbf{1.55(1.66)}\\
        \midrule
        \textbf{Relative $\Delta\%$ (HSCCAF vs. GCN)} & +1.45\% & +1.46\% & +1.80\% & +38.02\% & +71.56\% \\
        
        \textbf{Relative $\Delta\%$ (HSCCAF vs. SAGE)} & +3.56\% & +2.29\% & +4.87\% & +15.57\% & +71.87\% \\
        
        \textbf{Relative $\Delta\%$ (HSCCAF vs. GIN)} & +7.45\% & +5.74\% & +10.22\% & +39.67\% & +73.86\% \\
        
        \textbf{Relative $\Delta\%$ (HSCCAF vs. FairGNN)} & +8.12\% & +12.8\% & +19.51\% & -70.90\% & +57.18\% \\
        
        \textbf{Relative $\Delta\%$ (HSCCAF vs. FairVGNN)} & +7.43\% & +5.47\% & +9.85\% & -16.04\% & +43.01\% \\
        
        \textbf{Relative $\Delta\%$ (HSCCAF vs. CAF)} & -1.06\% & -1.47\% & -1.37\% & +11.73\% & +24.75\% \\
        
        \textbf{Relative $\Delta\%$ (HSCCAF vs. FairGB)} & +3.38\% & +1.81\% & +4.47\% & -35.58\% & +56.94\% \\
        
         \textbf{Relative $\Delta\%$ (HSCCAF vs. FairSAD)} & +5.94\% & +5.24\% & +7.82\% & -24.77\% & +49.18\% \\
        \bottomrule
    \end{tabular}
    }
\end{table}

On the Credit dataset (Table~\ref{tab:Credit}), HSCCAF again achieves the lowest values for both fairness metrics ($\Delta_{SP}=2.93$, $\Delta_{EO}=2.05$), outperforming CAF and all other baselines by a large margin in terms of fairness. Although it does not attain the highest classification scores (FairGB leads slightly), HSCCAF remains competitive, with a balanced accuracy of 69.10\% and an AUC of 73.41\% and F1 of 85.38\%. These results confirm that HSCCAF successfully reduces bias without substantially compromising predictive accuracy.

\begin{table}[h]
    \centering
    \caption{Evaluation of Node Classification and Fairness Performance on the Credit Dataset.}\label{tab:Credit}
    \resizebox{0.80\textwidth}{!}{%
    \begin{tabular}{lccccc}
        \toprule
        \textbf{Method} & \textbf{BACC($\uparrow$)} & \textbf{AUC($\uparrow$)} & \textbf{F1($\uparrow$)} & \boldmath$\Delta_{SP}(\downarrow)$ & \boldmath$\Delta_{EO}(\downarrow)$ \\
        \midrule
        \textbf{GCN} & 64.81(0.33) & 70.13(0.28) & 79.18(1.57) & 7.49(3.70) & 6.35(3.80)\\
       
        \textbf{SAGE} & \uline{69.54(0.72)} & \textbf{75.57(0.32)} & 83.04(3.75) & 6.36(3.34) & 4.86(3.58) \\
       
        \textbf{GIN} & 65.83(0.20) & 70.59(0.48) & 79.13(2.10) & 10.70(4.13) & 9.93(4.04) \\
       
        \textbf{FairGNN(GCN)} & 64.14(0.45) & 70.28(0.12) & 77.19(1.22) & 10.80(3.16) & 9.94(3.32) \\
       
        \textbf{FairVGNN(GCN)} & 65.56(0.13) & 70.20(0.23) & \textbf{85.97(0.16)} & 8.90(0.88) & 7.01(0.89) \\
       
        \textbf{CAF(SAGE)} & 67.35(2.90) &73.45(2.29) & 75.46(6.84) & 4.95(2.71) & 4.65(2.45) \\
        
        \textbf{FairGB(SAGE)} & \textbf{69.58(0.40)} & \uline{73.53(0.56)} & 82.84(1.56) & 3.81(3.27) & 2.83(2.94) \\
        
        \textbf{FairSAD(GCN)} & 63.21(0.41) & 67.60(0.82) & 78.09(0.99) & \uline{3.42(2.04)} & \uline{2.72(2.25)} \\
        
        \textbf{HSCCAF(SAGE)} & 69.10(0.34) & 73.41(1.17) & \uline{85.38(0.59)} & \textbf{2.93(1.35)} & \textbf{2.05(1.08)} \\
        \midrule
       \textbf{Relative $\Delta\%$ (HSCCAF vs. GCN)} & +6.62\% & +4.67\% & +7.83\% & +60.88\% & +67.71\% \\
       
        \textbf{Relative $\Delta\%$ (HSCCAF vs. SAGE)} & +0.63\% & -2.86\% & +2.81\% & +53.93\% & +57.81\% \\
        
        \textbf{Relative $\Delta\%$ (HSCCAF vs. GIN)} & +4.96\% & +3.99\% & +7.90\% & +72.61\% & +79.35\% \\
        
        \textbf{Relative $\Delta\%$ (HSCCAF vs. FairGNN)} & +7.73\% & +4.45\% & +10.61\% & +72.87\% & +79.37\% \\
        
        \textbf{Relative $\Delta\%$ (HSCCAF vs. FairVGNN)} & +5.40\% & +4.57\% & -0.68\% & +67.07\% & +70.75\% \\
        
        \textbf{Relative $\Delta\%$ (HSCCAF vs. CAF)} & +2.60\% & +0.05\% & +13.14\% & +40.80\% & +55.91\% \\
        
        \textbf{Relative $\Delta\%$ (HSCCAF vs. FairGB)} & -0.69\% & +0.16\% & +3.06\% & +23.09\% & +27.56\% \\
        
        \textbf{Relative $\Delta\%$ (HSCCAF vs. FairSAD)} & +9.31\% & +8.59\% & +9.33\% & +14.32\% & +24.63\% \\
        \bottomrule
    \end{tabular}
    }
\end{table}

On the NBA dataset (Table~\ref{tab:NBA}), HSCCAF (SAGE) delivers the best fairness outcomes, with $\Delta_{SP}=5.16$ and $\Delta_{EO}=4.65$, markedly reducing bias compared to other methods. In terms of predictive performance, CAF(SAGE) achieves the highest balanced accuracy (67.89\%), while FairVGNN(GCN) yields the strongest F1-score (85.97). However, both methods exhibit considerably larger fairness gaps. By contrast, HSCCAF provides a more balanced trade-off, maintaining solid classification performance (BACC = 67.59\%, AUC = 74.64\%, F1 = 74.48\%) while substantially improving fairness. These results demonstrate HSCCAF’s effectiveness in mitigating unfairness while preserving reasonable predictive accuracy on this challenging dataset.

\begin{table}[h]
    \centering
    \caption{Evaluation of Node Classification and Fairness Performance on the NBA Dataset.}\label{tab:NBA}
    \resizebox{0.80\textwidth}{!}{%
    \begin{tabular}{lccccc}
        \toprule
        \textbf{Method} & \textbf{BACC($\uparrow$)} & \textbf{AUC($\uparrow$)} & \textbf{F1($\uparrow$)} & \boldmath$\Delta_{SP}(\downarrow)$ & \boldmath$\Delta_{EO}(\downarrow)$ \\
        \midrule
        \textbf{GCN} & 65.22(5.87) & 70.56(5.94) & 67.78(5.73) & \uline{6.13(2.63)} & 5.75(8.94)\\
        
        \textbf{SAGE} & 65.76(4.69) &74.58(3.45) & 67.51(3.63) & 9.20(9.46) & 12.86(5.70) \\
        
        \textbf{GIN} & 65.83(5.42) & \textbf{75.46(5.60)} & 65.47(6.54) & 13.99(10.40) & 19.99(13.30)\\
        
        \textbf{FairGNN(GCN)} & 62.35(6.26) & 70.63(6.25) & 67.12(7.01) & 8.09(4.68) & 8.76(4.56) \\
        
        \textbf{FairVGNN(GCN)} & 65.56(0.13) & 70.20(0.23) & \textbf{85.97(0.16)} & 8.09(4.68) & 8.76(4.56) \\
        
        \textbf{CAF(SAGE)} & \textbf{67.89(3.95)} & 73.55(5.86) & 66.87(4.92) & 8.51(6.40) & \textbf{4.34(4.38)} \\
        
        \textbf{FairGB(SAGE)} &  67.02(3.23) & 72.68(2.76) & 69.01(2.98) & 8.00(7.10) & 10.10(7.30) \\
        
         \textbf{FairSAD(GCN)} &  61.62(6.32) & 68.41(4.22) & 71.51(2.66) & 6.56(4.86) & 6.48(3.75) \\
        
        \textbf{HSCCAF(SAGE)} & \uline{67.59(2.17)} &  \uline{74.64(1.02)} & \uline{74.48(7.18)} & \textbf{5.16(3.88)} & \uline{4.65(3.81)} \\
        \midrule
       \textbf{Relative $\Delta\%$ (HSCCAF vs. GCN)} & +3.63\% & +5.78\% & +9.89\% & +15.82\% & +19.13\% \\
       
        \textbf{Relative $\Delta\%$ (HSCCAF vs. SAGE)} & +2.78\% & +0.09\% & +10.32\% & +43.91\% & +63.84\% \\
        
        \textbf{Relative $\Delta\%$ (HSCCAF vs. GIN)} & +2.67\% & -1.08\% & +13.76\% & +63.11\% & +76.73\% \\
        
        \textbf{Relative $\Delta\%$ (HSCCAF vs. FairGNN)} & +8.40\% & +5.67\% & +10.96\% & +36.21\% & +46.91\% \\
        
        \textbf{Relative $\Delta\%$ (HSCCAF vs. FairVGNN)} & +3.09\% & +6.32\% & -13.36\% & +36.21\% & +46.91\% \\
        
        \textbf{Relative $\Delta\%$ (HSCCAF vs. CAF)} & -0.44\% & +1.48\% & +11.36\% & +39.36\% & -7.14\% \\
        
        \textbf{Relative $\Delta\%$ (HSCCAF vs. FairGB)} & +0.85\% & +2.69\% & +7.92\% & +35.50\% & +53.96\% \\
        
        \textbf{Relative $\Delta\%$ (HSCCAF vs. FairSAD)} & +9.70\% & +9.06\% & +4.15\% & +21.34\% & +28.24\% \\
        \bottomrule
    \end{tabular}
    }
\end{table}

On the Pokec-n dataset (Table~V), HSCCAF (SAGE) achieves the strongest fairness performance among all methods, with $\Delta_{SP} = 1.31$ and $\Delta_{EO} = 1.71$. This represents a significant improvement over both vanilla GNNs and other fairness-aware approaches. At the same time, HSCCAF maintains reasonable classification performance, achieving a BACC of 64.33\% and an AUC of 67.64\%. Although FairGNN (GCN) attains the best classification results, with the highest balanced accuracy (66.72\%), AUC (71.42\%), and F1-score (63.20\%), it does so at the cost of larger fairness gaps. These results confirm that HSCCAF effectively reduces bias while preserving reasonable predictive performance.

\begin{table}[h]
    \centering
    \caption{Evaluation of Node Classification and Fairness Performance on the Pokec-n Dataset.}\label{tab:pocke}
    \resizebox{0.80\textwidth}{!}{%
    \begin{tabular}{lccccc}
        \toprule
        \textbf{Method} & \textbf{BACC($\uparrow$)} & \textbf{AUC($\uparrow$)} & \textbf{F1($\uparrow$)} & \boldmath$\Delta_{SP}(\downarrow)$ & \boldmath$\Delta_{EO}(\downarrow)$ \\
        \midrule
        \textbf{GCN} & 64.65(1.24) & 69.21(1.40) & 61.46(3.26) & 1.85(1.01) & 2.37(1.58)\\
        
        \textbf{SAGE} & 64.36(2.44) & 69.44(2.31) & 60.50(4.27) & \uline{1.56(0.96)} & 2.24(1.95) \\
        
        \textbf{GIN} & 65.75(1.54) &  69.89(1.35) & 59.34(1.53) & 1.92(0.96) & 2.44(1.66)\\
        
        \textbf{FairGNN(GCN)} & \textbf{66.72(1.46)} & \textbf{71.42(0.47)} & \textbf{63.20(1.63)} & 3.38(2.24) & 5.59(1.88) \\
        
        \textbf{FairVGNN(GCN)} & \uline{66.38(1.21)} & \uline{70.17(0.45)} & 58.91(3.91) & 4.23(1.42) & 5.44(2.51) \\
        
        \textbf{CAF(SAGE)} & 62.43(1.65) & 65.76(1.80) & \uline{61.46(3.22)} & 1.64(0.53) & \uline{1.78(1.81)} \\
        
        \textbf{FairGB(SAGE)} &  65.08(2.45) & 70.11(1.60) & 61.15(5.91) & 2.13(0.60) & 2.76(1.01) \\
        
        \textbf{FairSAD(GCN)} &  64.77(1.85) & 68.78(1.61) & 59.40(3.28) & 2.94(1.93) & 4.57(3.07) \\
        
        \textbf{HSCCAF(SAGE)} & 64.33(0.72) & 67.64(1.14) & 59.28(1.90) & \textbf{1.31(0.44)} & \textbf{1.71(0.56)} \\
        \midrule
       \textbf{Relative $\Delta\%$ (HSCCAF vs. GCN)} & -0.49\% & -2.26\% & -3.54\% & +29.18\% & +27.84\% \\
       
        \textbf{Relative $\Delta\%$ (HSCCAF vs. SAGE)} & -0.04\% & -2.59\% & -2.01\% & +16.02\% & +23.66\% \\
        
        \textbf{Relative $\Delta\%$ (HSCCAF vs. GIN)} & -2.16\% & -3.22\% & -0.10\% & +31.77\% & +29.91\% \\
        
        \textbf{Relative $\Delta\%$ (HSCCAF vs. FairGNN)} & -3.58\% & -5.29\% & -6.20\% & +61.24\% & +69.40\% \\
        
        \textbf{Relative $\Delta\%$ (HSCCAF vs. FairVGNN)} & -3.09\% & -3.61\% & +0.63\% & +69.03\% & +68.56\% \\
        
        \textbf{Relative $\Delta\%$ (HSCCAF vs. CAF)} & +3.05\% & +2.86\% & -3.54\% & +20.12\% & +3.93\% \\
        
        \textbf{Relative $\Delta\%$ (HSCCAF vs. FairGB)} & -1.15\% & -3.52\% & -3.07\% & +38.49\% & +38.04\% \\
        
        \textbf{Relative $\Delta\%$ (HSCCAF vs. FairSAD)} & -0.67\% & -1.65\% & -0.20\% & +55.44\% & +62.58\% \\
        \bottomrule
    \end{tabular}
    }
\end{table}

In summary, HSCCAF demonstrates robust and consistent improvements in fairness across all datasets, while maintaining classification performance close to the best models. Compared to CAF, on which it builds, HSCCAF achieves lower fairness gaps in all scenarios. These results validate the contribution of HSCCAF’s design, which introduces fairness-enhancing mechanisms without degrading performance. The overall trend highlights that our approach leads to a better fairness-accuracy trade-off than both standard GNNs and existing fairness-aware baselines.
\subsection{Ablation Study}
In this section, we investigate the influence of graph editing, supervised contrastive loss, and environmental loss separately.
\paragraph{Effect of the graph editing}  
We investigate the influence of graph editing by comparing the graph homophily ratio with respect to the class label and the sensitive attribute label in the original and edited graphs. We expect the homophily ratio with respect to the class labels in the original graph to be higher than in the edited graph, while the homophily ratio with respect to the sensitive attribute label in the original graph is expected to be lower than in the edited graph. As shown in Table~\ref{tab:homophily}, this observation holds in all datasets.
\begin{table}[htbp]
\centering
\caption{Homophily ratios with respect to sensitive attributes and class labels before and after graph editing.}
\label{tab:homophily}
 \resizebox{0.55\textwidth}{!}{%
\begin{tabular}{lcccc}
\toprule
\textbf{Dataset} & $\boldsymbol{hr}_{s}(\text{original})$ & $\boldsymbol{hr}_{s}(\text{edited})$ & $\boldsymbol{hr}_{c}(\text{original})$ & $\boldsymbol{hr}_{c}(\text{edited})$ \\
\midrule
\textbf{German} & 0.80 & 0.74 ($\downarrow$) & 0.59 & 0.62 ($\uparrow$) \\
\textbf{Bail}   & 0.53 & 0.49 ($\downarrow$) & 0.78 & 0.79 ($\uparrow$) \\
\textbf{Credit} & 0.89   & 0.86 ($\downarrow$)  & 0.67   & 0.69 ($\uparrow$)  \\   
\textbf{NBA}   & 0.72 & 0.59 ($\downarrow$) & 0.40 & 0.48 ($\uparrow$) \\
\textbf{Pokec-n}   & 0.95 & 0.93 ($\downarrow$) & 0.75 & 0.76 ($\uparrow$) \\
\bottomrule
\end{tabular}
}
\end{table}

 Additionally, we compare the results of HSCCAF with and without the Fairness-Aware Graph Editing phase across all datasets. As shown in Table~\ref{tab:HSCCAF}, the results without this phase are consistently worse than those with it, except for the bail dataset. In the bail dataset, as shown in Table~\ref{tab:homophily}, the homophily ratio with respect to the sensitive attribute label is 0.53\%, indicating that there is little topology bias in the graph structure. Consequently, the results with and without the Fairness-Aware Graph Editing phase are similar, with HSCCAF showing only a slight improvement in the equality of opportunity fairness metric. This demonstrates the effectiveness of the Fairness-Aware Graph Editing phase in balancing fairness and accuracy when there is a topology bias in the graph structure. We also apply the Fairness-Aware Graph Editing phase to the CAF method to investigate its effect, as shown in Table~\ref{tab:HSCCAF_CAF}. Although this phase helps reduce topology bias in the graph structure, applying it to CAF alone does not achieve the best reported results. As shown in Table~\ref{tab:HSCCAF_CAF}, while there are some improvements in certain datasets, the performance is still not as good as HSCCAF. This highlights the importance of the additional losses introduced in HSCCAF, namely the supervised contrastive loss and the environmental loss. HSCCAF not only addresses the challenges tackled by CAF but also reduces topology bias in the graph structure, which CAF does not handle.

\begin{table}[h]
    \centering
    \caption{Evaluation of Node Classification and Fairness Performance of HSCCAF and HSCCAF(Without G.E.) (without Fairness-Aware Graph Editing phase)}
    \label{tab:HSCCAF}
    \resizebox{0.75\textwidth}{!}{%
    \begin{tabular}{lccccc}
        \toprule
        \textbf{Method} & \textbf{BACC($\uparrow$)} & \textbf{AUC($\uparrow$)} & \textbf{F1($\uparrow$)} & \boldmath$\Delta_{SP}(\downarrow)$ & \boldmath$\Delta_{EO}(\downarrow)$ \\
        \midrule
        \multicolumn{6}{c}{\textbf{German}} \\ 
        \midrule
        \textbf{HSCCAF(Without G.E.)} & 57.64(4.53) & 62.56(4.60) & 61.51(5.37) & 5.35(2.58) & 8.81(1.88) \\
        \textbf{HSCCAF(With G.E.)} & \textbf{60.02(2.93)} & \textbf{63.74(3.87)} & \textbf{65.50(4.70)} & \textbf{2.86(2.20)} & \textbf{3.39(1.05)} \\
        \midrule
        \multicolumn{6}{c}{\textbf{Bail}} \\
        \midrule
        \textbf{HSCCAF(Without G.E.)} & \textbf{84.78(0.94)} & \textbf{91.14(0.89)} & 81.00(1.17) & \textbf{5.13(2.13)} & 1.95(0.96) \\
        \textbf{HSCCAF(With G.E.)} & 84.67(1.50) & 90.49(0.90) & \textbf{81.03(1.99)} & 5.64(1.89) & \textbf{1.55(1.66)} \\
        \midrule
        \multicolumn{6}{c}{\textbf{Credit}} \\
        \midrule
        \textbf{HSCCAF(Without G.E.)} & 69.03(0.61) & \textbf{74.70(0.53)} & 81.98(2.76) & 6.91(3.24) & 6.07(3.71) \\
        \textbf{HSCCAF(With G.E.)} & \textbf{69.10(0.34)} & 73.41(1.17) & \textbf{85.38(0.59)} & \textbf{2.93(1.35)} & \textbf{2.05(1.08)} \\
      
        \midrule
        \multicolumn{6}{c}{\textbf{NBA}} \\
        \midrule
        \textbf{HSCCAF(Without G.E.)} & 67.04(4.31) & 72.21(5.66) & 68.00(5.42) & 5.33(4.74) & 8.24(7.30) \\
        \textbf{HSCCAF(With G.E.)} & \textbf{67.59(2.17)} & \textbf{74.64(1.02)} & \textbf{74.48(7.18)} & \textbf{5.16(3.88)} & \textbf{4.65(3.81)} \\

        \midrule
        \multicolumn{6}{c}{\textbf{Pokec-n}} \\
        \midrule
        \textbf{HSCCAF(Without G.E.)} & 63.11(1.04) & 66.54(3.41) & \textbf{61.53(2.42)} & 1.52(0.67) & 1.93(1.07) \\
        \textbf{HSCCAF(With G.E.)} & \textbf{64.33(0.72)} & \textbf{67.64(1.14)} & 59.28(1.90) & \textbf{1.31(0.44))} & \textbf{1.71(0.56)} \\
        
        \bottomrule
    \end{tabular}
    }
\end{table}

\begin{table}[h]
    \centering
    \caption{Evaluation of Node Classification and Fairness Performance of CAF \cite{guo2023towards} and CAF (With G.E.) (with Fairness-Aware Graph Editing phase)}
    \label{tab:HSCCAF_CAF}
    \resizebox{0.65\textwidth}{!}{%
    \begin{tabular}{lccccc}
        \toprule
        \textbf{Method} & \textbf{BACC($\uparrow$)} & \textbf{AUC($\uparrow$)} & \textbf{F1($\uparrow$)} & \boldmath$\Delta_{SP}(\downarrow)$ & \boldmath$\Delta_{EO}(\downarrow)$ \\
        \midrule
        \multicolumn{6}{c}{\textbf{German}} \\ 
        \midrule
         \textbf{CAF \cite{guo2023towards}} & 57.58(5.36) & 61.87(8.13) & 62.96(9.18) & \textbf{4.46(3.06)} & 6.15(0.85) \\
 \textbf{CAF(With G.E.)} & \textbf{59.70(2.81)} & \textbf{63.74(4.31)} & \textbf{66.31(4.56)} & 6.84(3.06) & \textbf{5.02(3.70)} \\
        \midrule
        \multicolumn{6}{c}{\textbf{Bail}} \\
        \midrule
        \textbf{CAF \cite{guo2023towards}} & \textbf{85.58(0.77)} & \textbf{91.84(0.87)} & \textbf{82.16(1.02)} & 6.39(2.52) & 2.06(1.81) \\
 \textbf{CAF(With G.E.)} & 84.77(1.47) & 90.87(0.82) & 81.23(1.96) & \textbf{5.54(2.27)} & \textbf{1.84(1.89)} \\
        \midrule
        \multicolumn{6}{c}{\textbf{Credit}} \\
        \midrule
   \textbf{CAF \cite{guo2023towards}} & 67.35(2.90) & 73.45(2.29) & 75.46(6.84) & 4.95(2.71) & 4.65(2.45) \\
 \textbf{CAF(With G.E.)} & \textbf{67.60(2.55)} & \textbf{73.82(2.15)} & \textbf{77.26(7.17)} & \textbf{3.55(2.86)} & \textbf{3.11(2.55)} \\
        \midrule
        \multicolumn{6}{c}{\textbf{NBA}} \\
        \midrule
       \textbf{CAF \cite{guo2023towards}} & 67.89(3.95) & \textbf{73.55(5.86)} & 66.87(4.92) & 8.51(6.40) & \textbf{4.34(4.38)} \\
 \textbf{CAF(With G.E.)} & \textbf{68.26(6.13)} & 71.72(6.42) & \textbf{70.94(4.91)} & \textbf{5.24(2.68)} & 8.30(5.09) \\

  \midrule
        \multicolumn{6}{c}{\textbf{Pokec-n}} \\
        \midrule
        \textbf{CAF \cite{guo2023towards}} & 62.43(1.65) & 65.76(1.80) & \textbf{61.46(3.22)} & 1.64(0.53) & \textbf{1.78(1.81)} \\
        \textbf{CAF(With G.E.)} & \textbf{63.12(1.47)} & \textbf{66.22(0.97)} & 60.31(3.75) & \textbf{1.46(0.81))} & 1.92(0.93) \\
        \bottomrule
    \end{tabular}
    }
\end{table}

\paragraph{Effect of the Contrastive Component}

We investigate the influence of the hyperparameter $\omega$, figures \ref{fig:omega}, which controls the relative importance of the content-based component in HSCCAF’s loss function. For the \textbf{German} dataset, we observe that increasing $\omega$ leads to improved fairness: both the demographic parity gap ($\Delta_{SP}$) and the equal opportunity gap ($\Delta_{EO}$) decrease as the content term gains more weight. This is a noteworthy outcome, as the content term was primarily introduced to enhance representation learning independently of group structure, rather than explicitly target fairness. Interestingly, this increase in fairness is also accompanied by gains in classification performance, with balanced accuracy, AUC, and F1-score all improving as $\omega$ increases. This suggests that the content term not only preserves, but can even enhance predictive utility while mitigating bias.

On the \textbf{Bail} dataset, the effect of $\omega$ is more nuanced. We again observe that higher values of $\omega$ contribute positively to fairness, though the improvement is more modest than in the German dataset. Predictive performance metrics remain stable for different values of $\omega$, indicating that emphasizing content does not degrade utility. In summary, the term centered on content controlled by $\omega$ appears to act as a stabilizing factor, helping to maintain classification performance while encouraging fairer representations.

\begin{figure*}[ht]
\centering
\begin{subfigure}[t]{0.40\textwidth}
\includegraphics[width=\textwidth]
{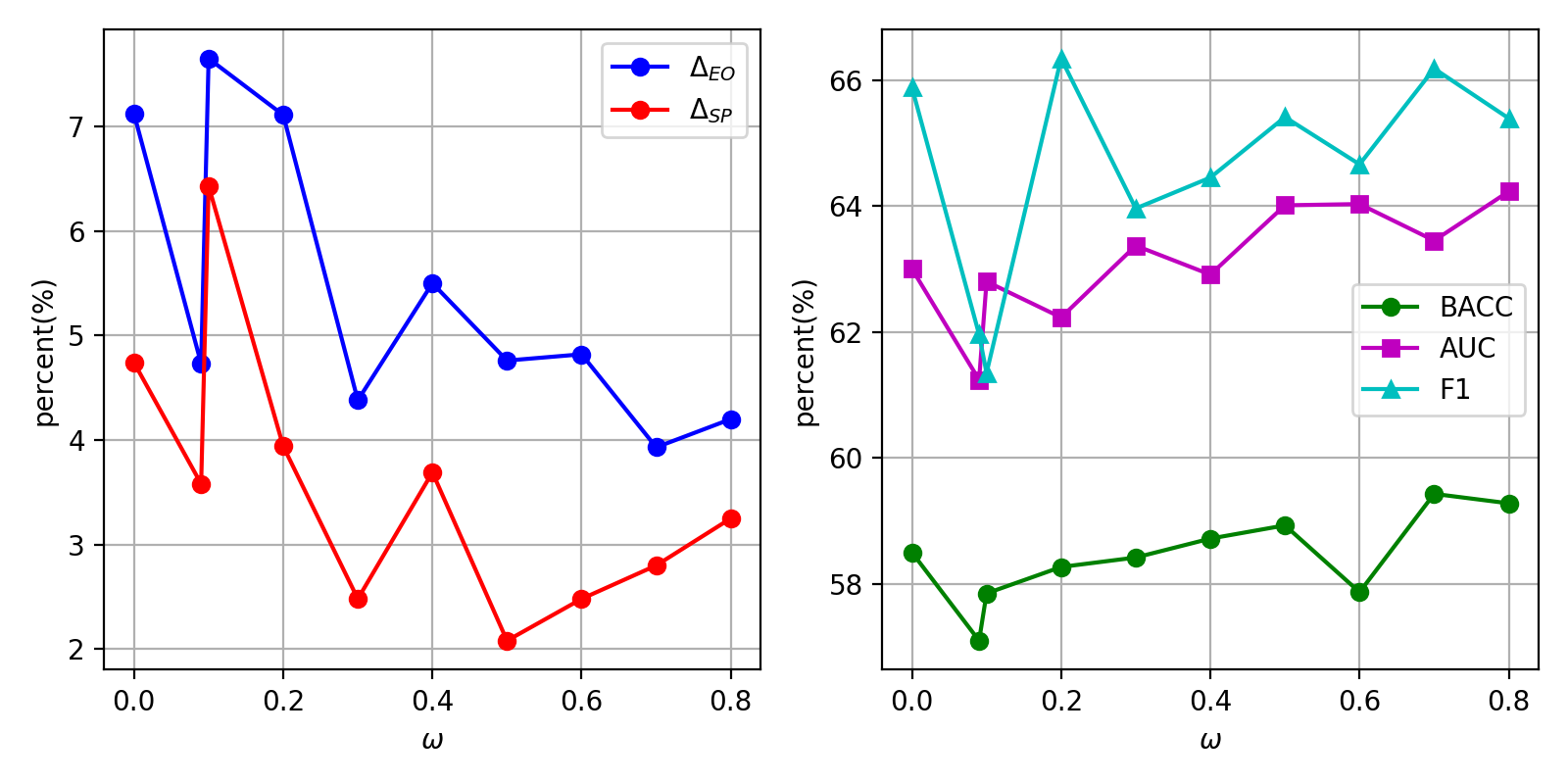}
\end{subfigure}
\begin{subfigure}[t]{0.40\textwidth}
\includegraphics[width=\textwidth]{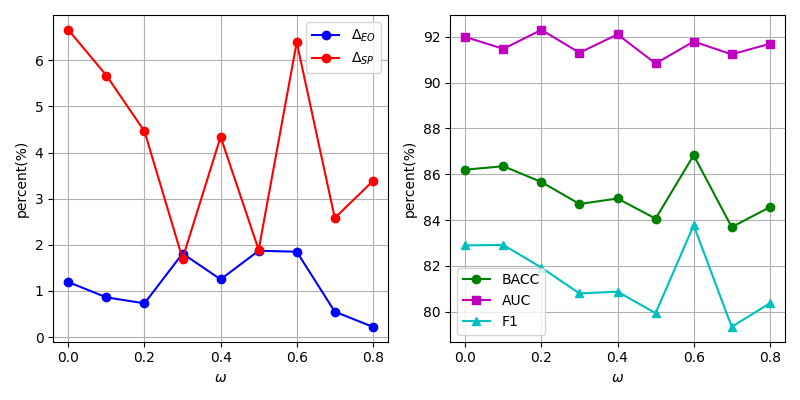}
\end{subfigure}
\caption{Hyper-parameter \(\omega\) study on the German (two plots on the right) and Bail (two plots on the left) datasets.}
\label{fig:omega}
\end{figure*}

\paragraph{Effect of the environment component}

We now turn to the study of $\eta$,  which determines the weight of the environmental loss in the total loss. Figure \ref{fig:eta} shows the utility and fairness performance as a function of $\eta$ while other parameters are fixed (with $\alpha = 10$, $\beta = 1$, $\gamma = 1$, and $\omega = 0$), using the test data from the German dataset. We notice that the impact of $\eta$ is less smooth to analyse than for $\omega$.

\begin{figure*}[ht]
\centering
\begin{subfigure}[t]{0.40\textwidth}
\includegraphics[width=\textwidth]{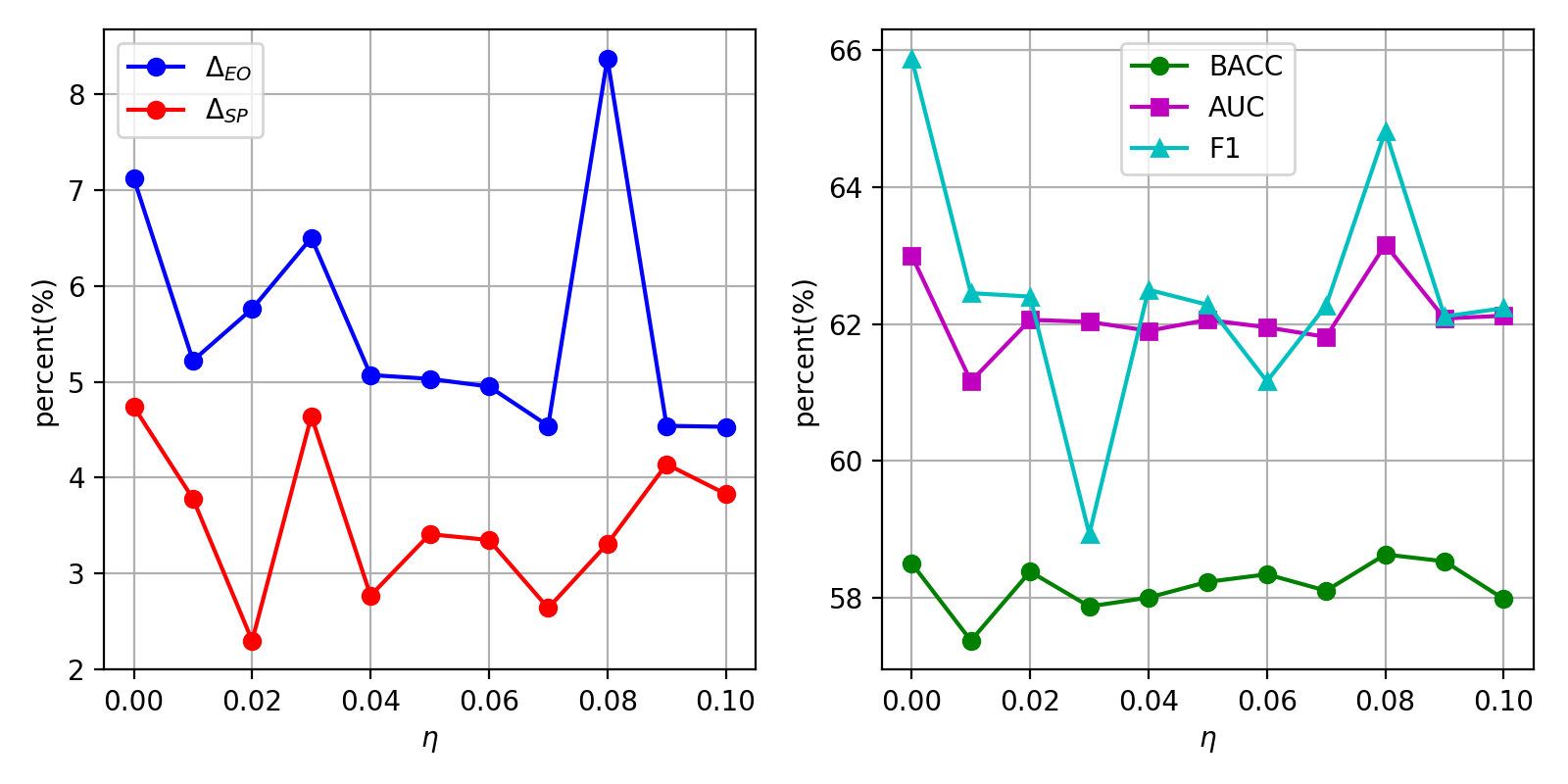}
\end{subfigure}
\begin{subfigure}[t]{0.40\textwidth}
\includegraphics[width=\textwidth]{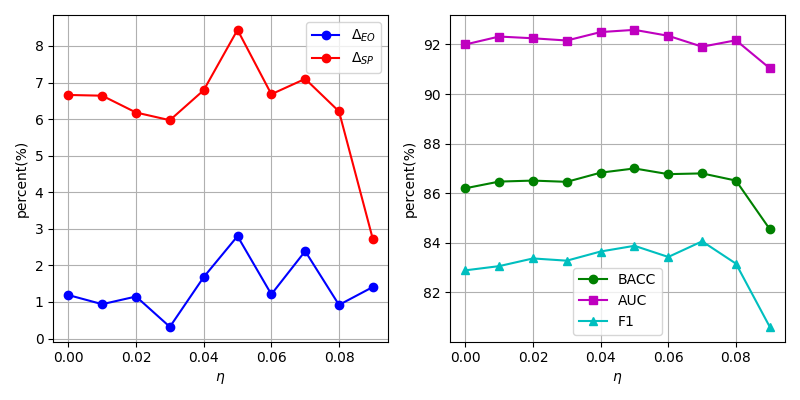}
\end{subfigure}
\caption{Hyper-parameter ( \(\eta\)) study on the German (two plots on the right) and Bail (two plots on the left) dataset.}
\label{fig:eta}
\end{figure*}

Additionally, Figures \ref{fig:tsne-german} and \ref{fig:tsne-bail} present t-SNE visualizations that compare the environment and content components extracted using the CAF and HSCCAF methods, respectively. In the case of CAF, although the environmental representation is theoretically designed to encode the sensitive attributes, the visualizations reveal that it does not accurately capture this information. Instead, the content representation appears to retain some level of entanglement with the sensitive attributes, suggesting that the method struggles to fully separate these representations. This failure highlights the challenge of achieving effective disentanglement using CAF alone.
\begin{figure*}[ht]
\centering
\begin{subfigure}[t]{0.42\textwidth}
\includegraphics[width=\textwidth]{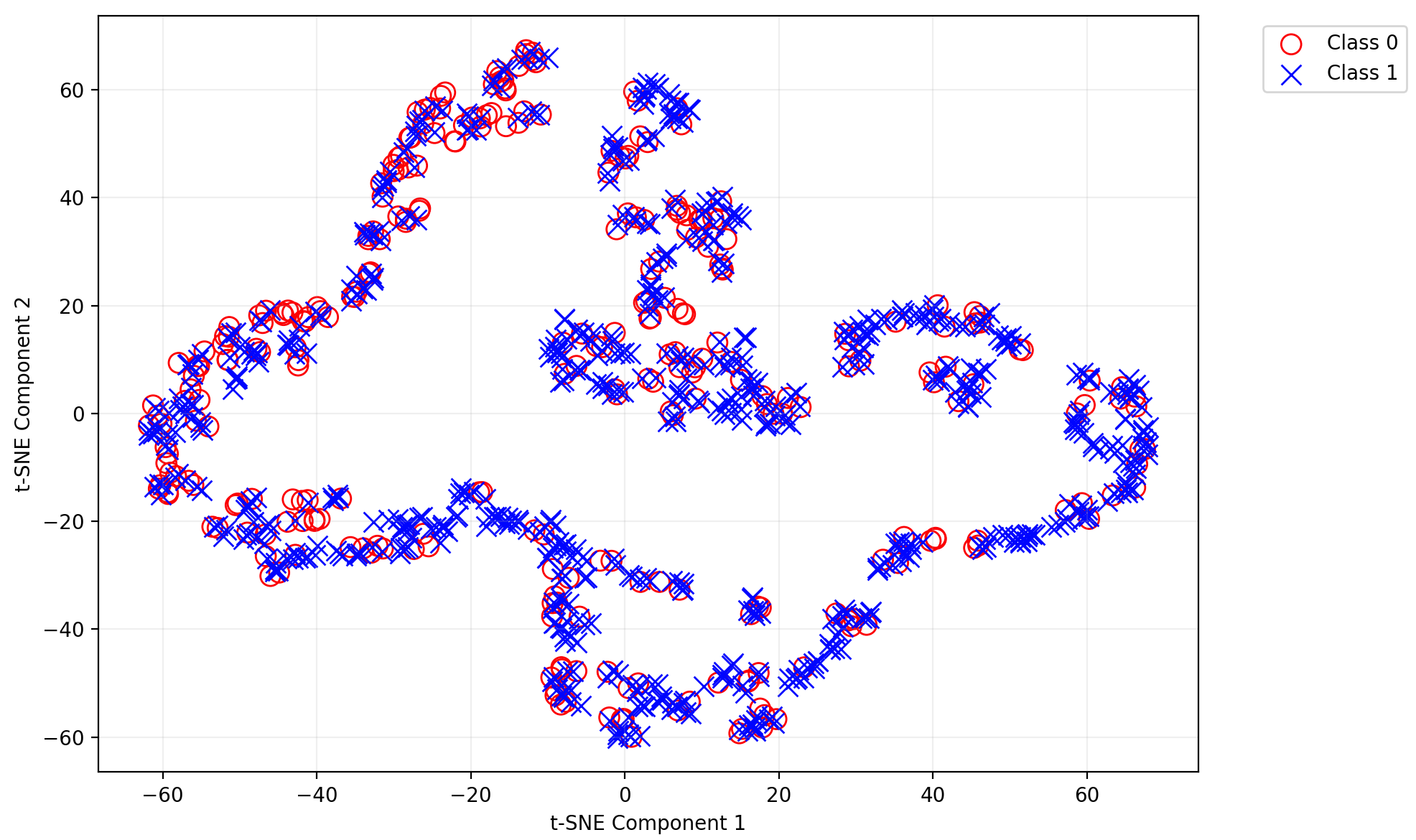}
\caption{CAF (labeled according to class labels)}
\end{subfigure}
\begin{subfigure}[t]{0.42\textwidth}
\includegraphics[width=\textwidth]{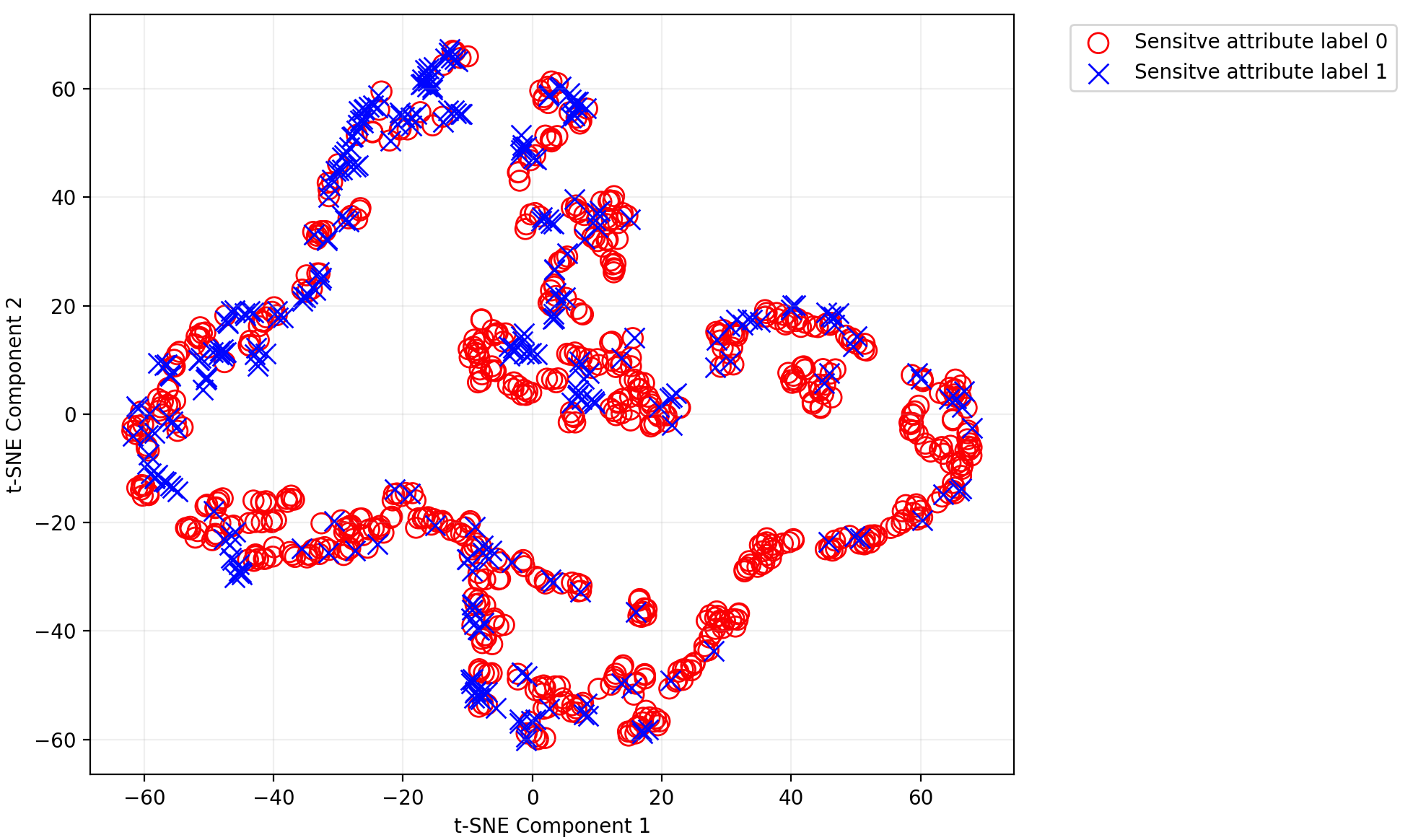}
\caption{CAF (labeled according to sensitive attribute labels)}
\end{subfigure}
\begin{subfigure}[t]{0.42\textwidth}
\includegraphics[width=\textwidth]{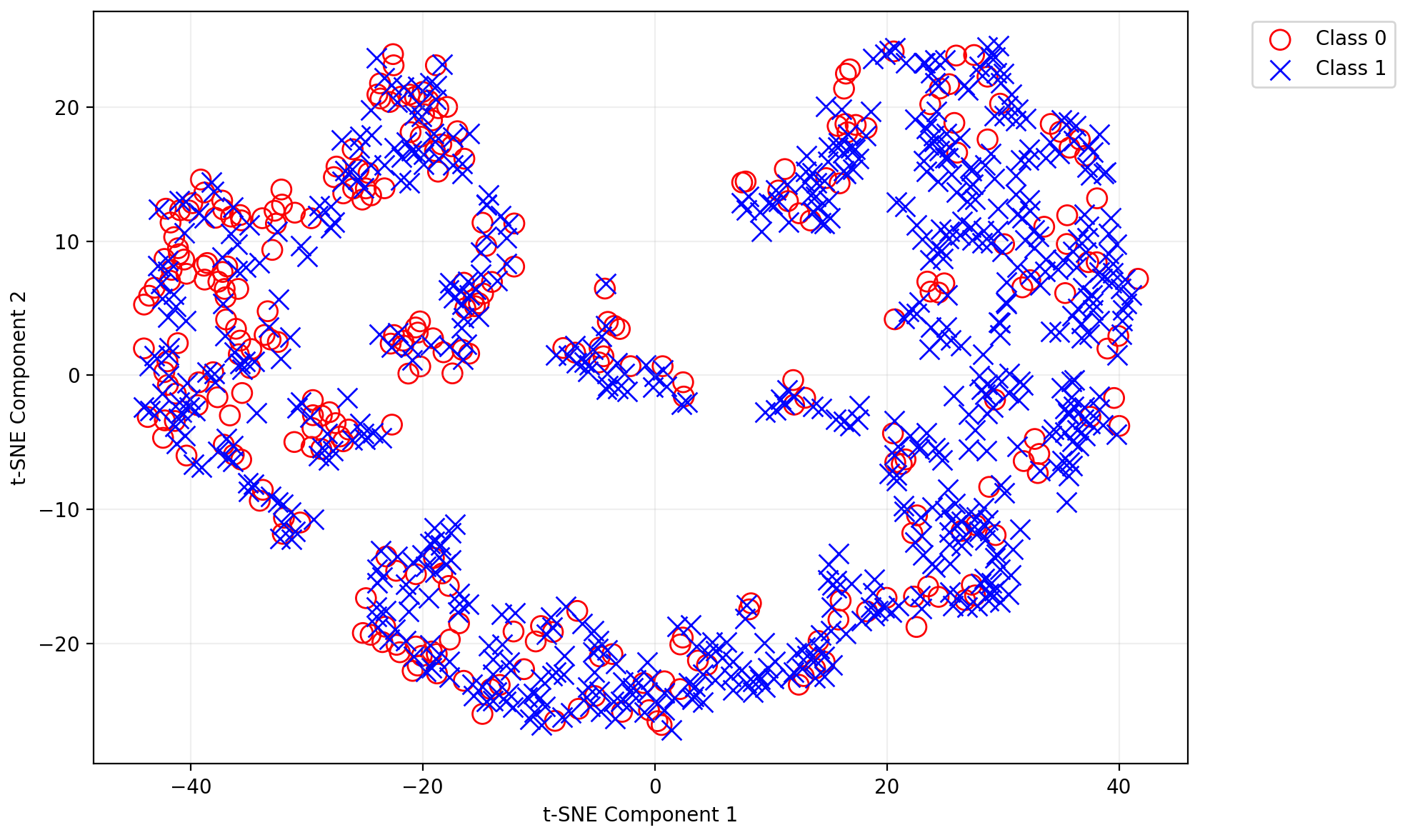}
\caption{HSCCAF (labeled according to class labels)}
\end{subfigure}
\begin{subfigure}[t]{0.42\textwidth}
\includegraphics[width=\textwidth]{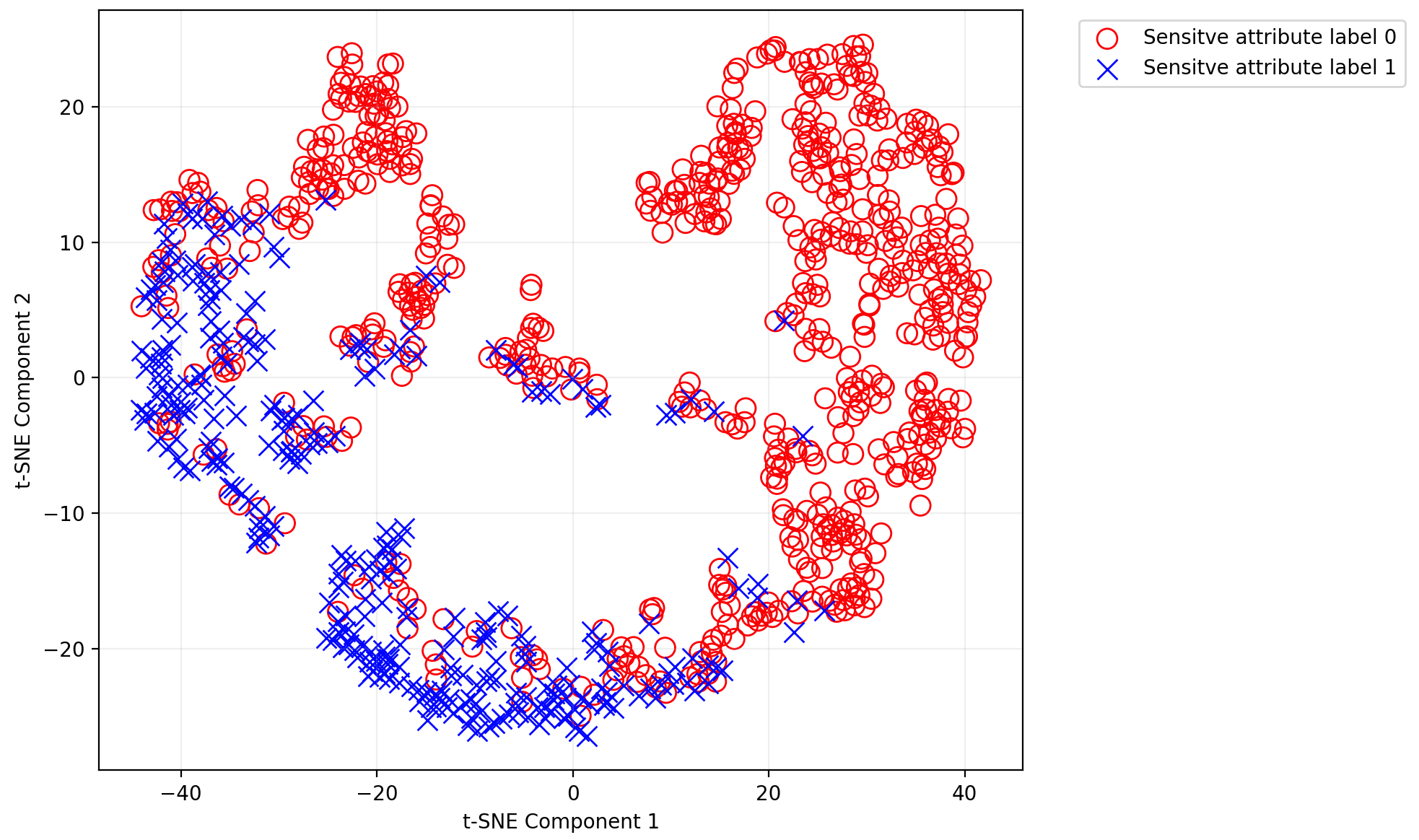}
\caption{HSCCAF (labeled according to sensitive attribute labels)}
\end{subfigure}
\caption{Projection of the embeddings in the Environment component using TSNE on the German dataset.}
\label{fig:tsne-german}
\end{figure*}

On the other hand, our proposed method, HSCCAF, demonstrates a significant improvement in this regard, especially for the German dataset. The t-SNE plots show that HSCCAF effectively isolates sensitive information in the environment component, while ensuring that the content component remains uninfluenced by sensitive attributes. Moreover, HSCCAF successfully preserves the information relevant to the class labels without exposing sensitive data. This behavior indicates that HSCCAF not only excels in disentangling sensitive and task-related information, but also does so in a way that maintains the utility of the model. By improving the separation of sensitive information, HSCCAF ensures fairness without sacrificing performance, providing a more robust solution to the problem of fairness in graph-based tasks.

\begin{figure*}[ht]
\centering
\begin{subfigure}[t]{0.42\textwidth}
\includegraphics[width=\textwidth]{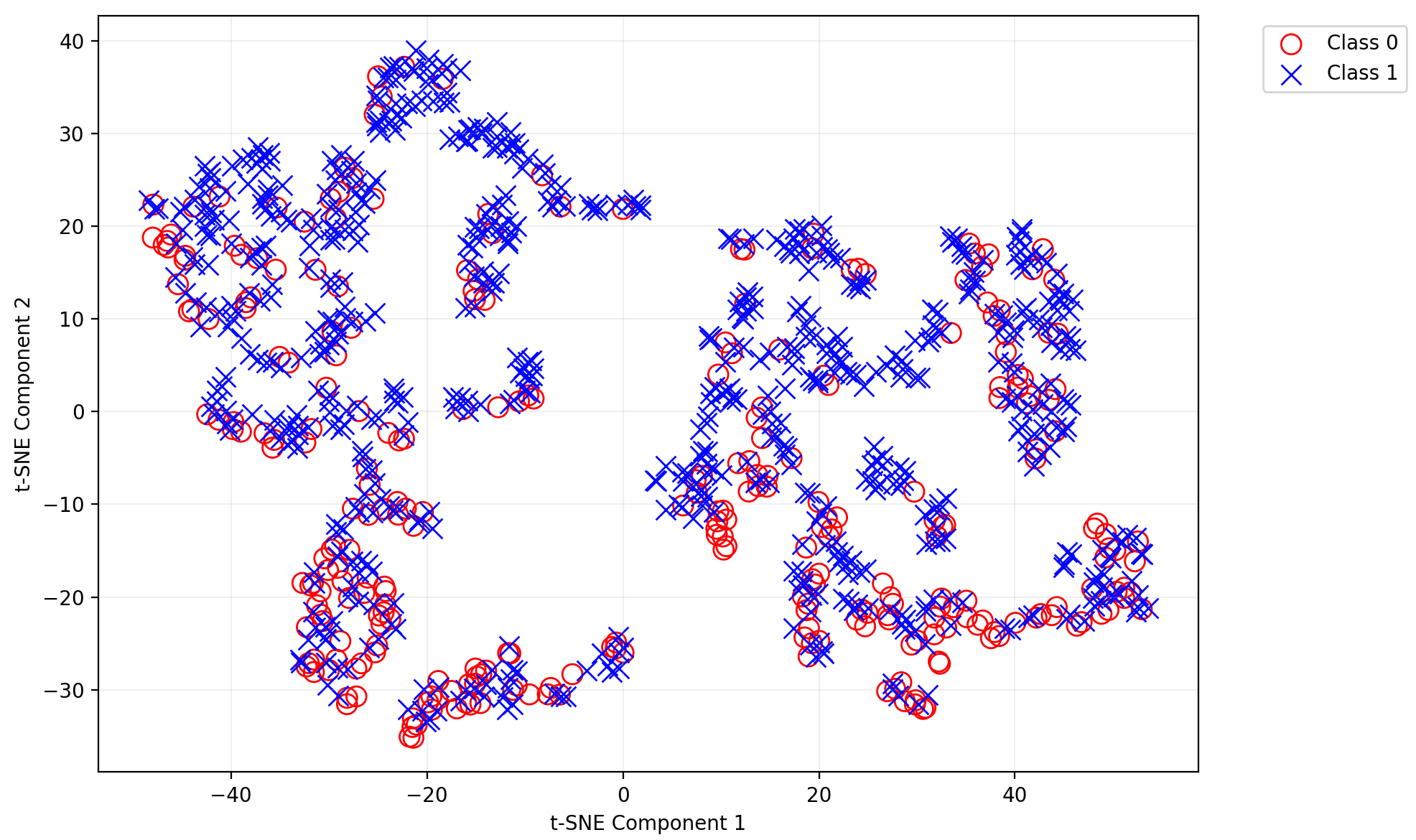}
\caption{CAF (labeled according to class labels)}
\end{subfigure}
\begin{subfigure}[t]{0.42\textwidth}
\includegraphics[width=\textwidth]{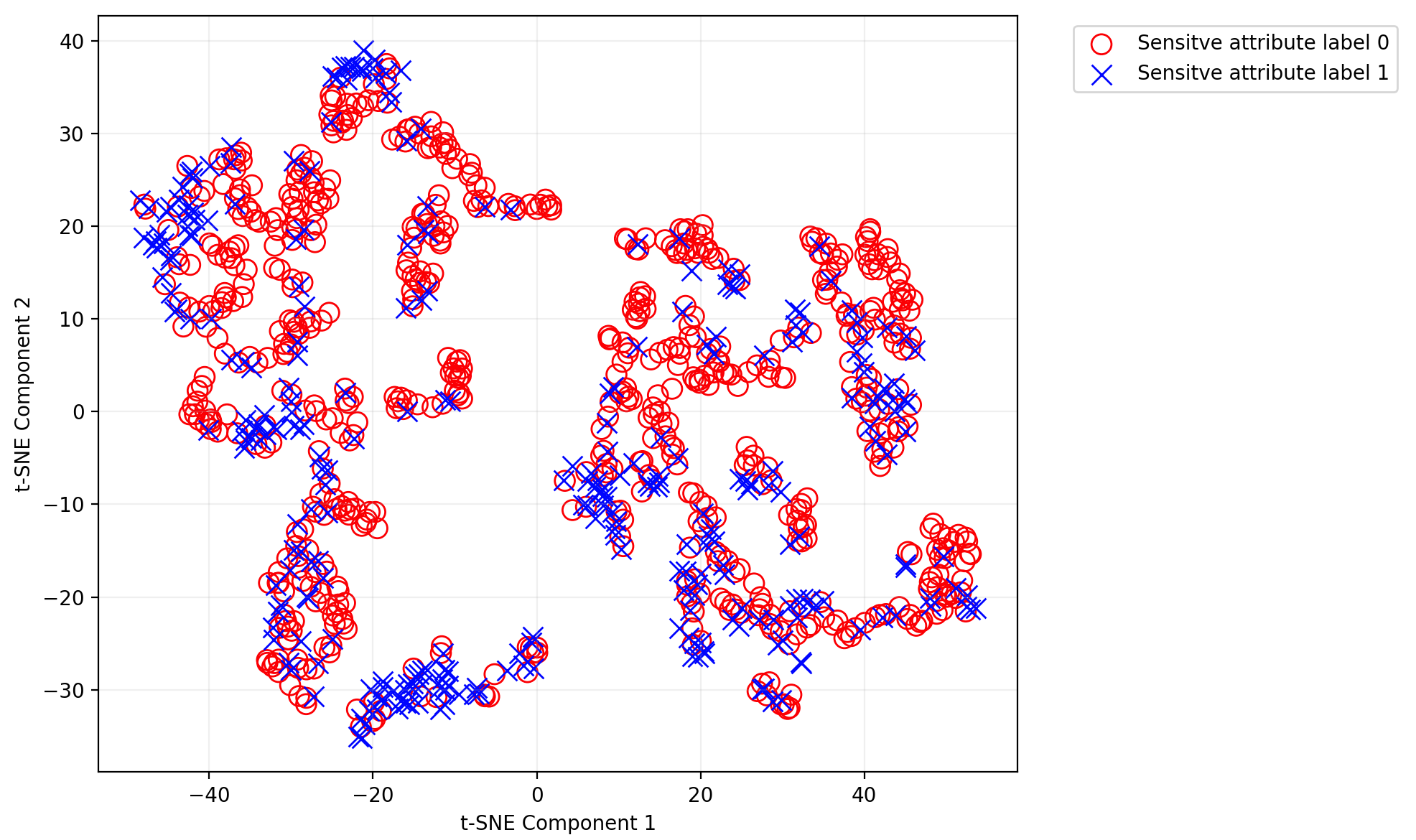}
\caption{CAF (labeled according to sensitive attribute labels)}
\end{subfigure}
\begin{subfigure}[t]{0.42\textwidth}
\includegraphics[width=\textwidth]{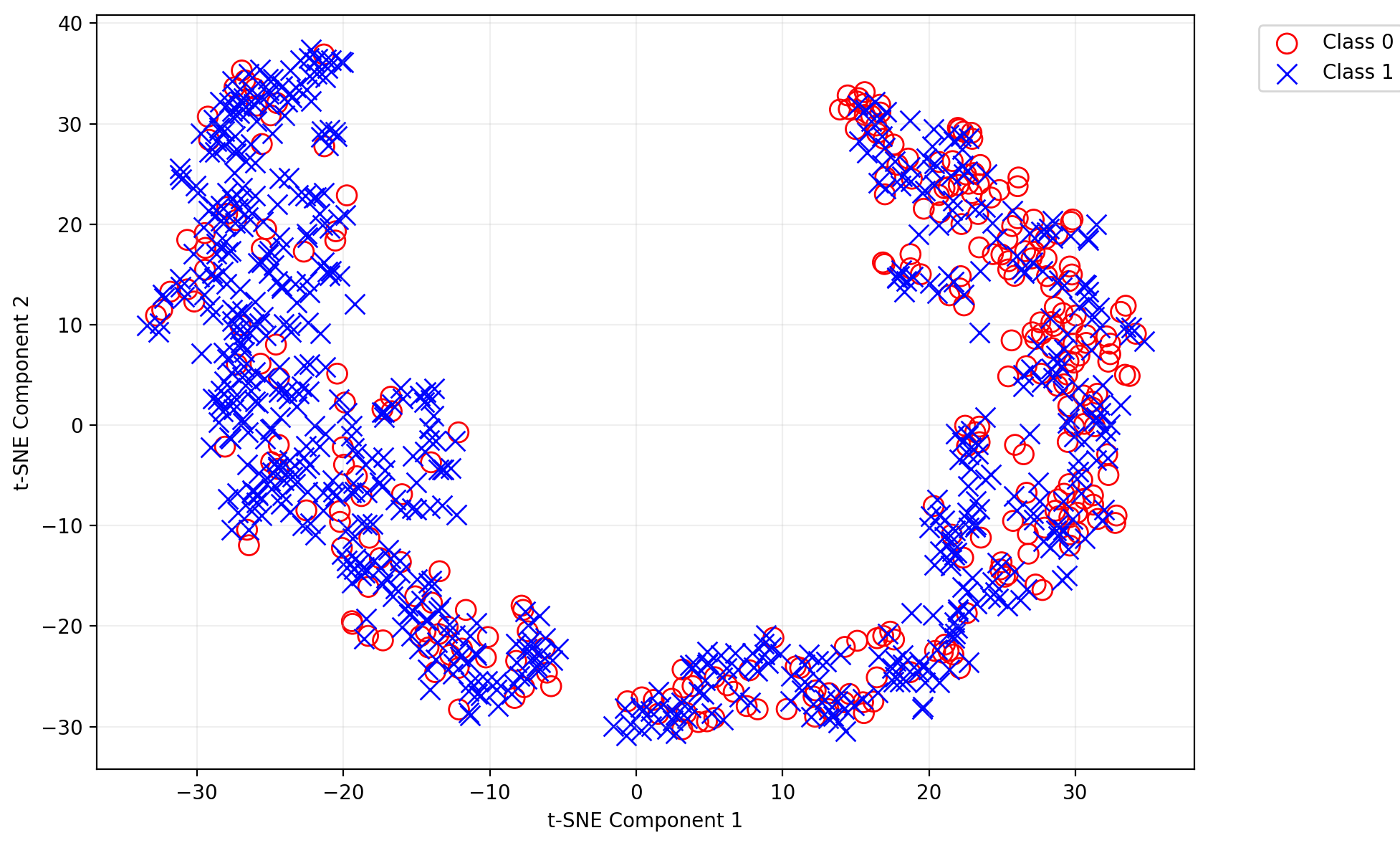}
\caption{HSCCAF (labeled according to class labels)}
\end{subfigure}
\begin{subfigure}[t]{0.42\textwidth}
\includegraphics[width=\textwidth]{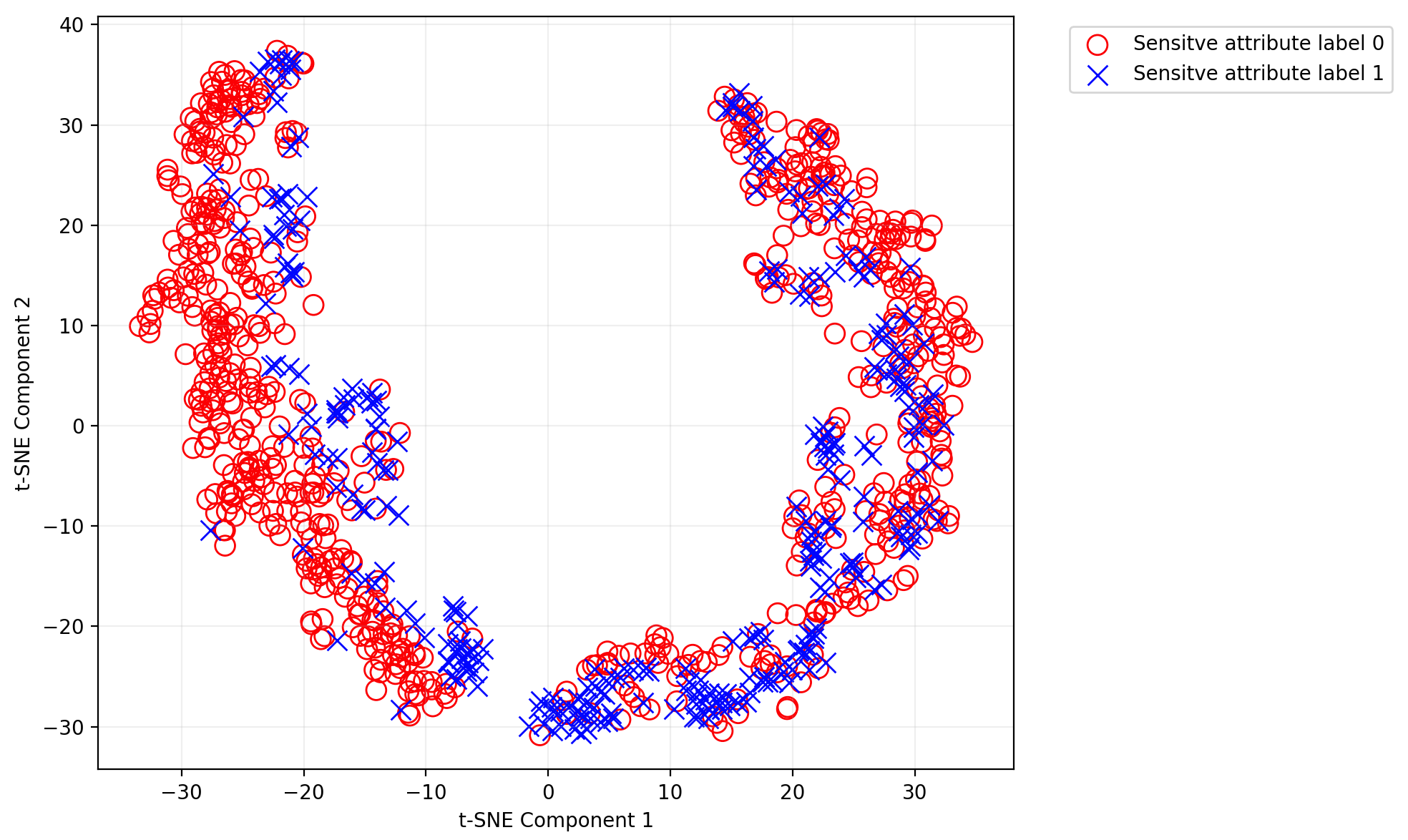}
\caption{HSCCAF (labeled according to sensitive attribute labels)}
\end{subfigure}
\caption{Projection of the embeddings in the Content component using TSNE on the German dataset.}
\label{fig:tsne-bail}
\end{figure*}
\paragraph{Best hyperparameters}
The optimal hyperparameters for the German, Bail, Credit, and NBA datasets, as shown in Table~\ref{hyper}, indicate distinct prioritizations for fairness and accuracy. A key observation is that the optimal value of $\beta$ is consistently equal to $1$ across all datasets. This consistency highlights the importance of preserving graph structure, ensuring that topological information is retained while other hyperparameters adjust to balance predictive accuracy, fairness, and environmental constraints.  

For the German dataset, the large value of $\alpha = 10$ emphasizes predictive accuracy, while $\gamma = 1$ enforces orthogonality between content and environment. The moderate values of $\omega = 0.3$ and $\eta = 0.09$ balance contrastive learning and environmental regularization.  

In the Bail dataset, the smaller $\alpha = 0.2$ reflects reduced emphasis on accuracy, while the very low $\gamma = 0.02$ indicates weak content--environment separation. The small values of $\omega = 0.03$ and $\eta = 0.07$ apply mild contrastive and environmental constraints.  

For the Credit dataset, $\alpha = 0.5$ provides a balance between accuracy and fairness. The relatively large $\omega = 0.7$ emphasizes contrastive learning, while $\gamma = 1$ and $\eta = 0.06$ ensure strong separation between content and environment with moderate environmental regularization.  

In the NBA dataset, $\alpha = 0.9$ again prioritizes predictive accuracy. The values $\gamma = 1$ and $\eta = 0.8$ highlight strong content-environment separation and environmental constraints, complemented by a small $\omega = 0.09$ that moderates the role of contrastive loss.

In the Pokec-n dataset, $\alpha = 0.2$ reflects reduced emphasis on accuracy, while $\gamma = 1$ enforces orthogonality between content and environment. The small values of $\omega = 0.09$ and $\eta = 0.1$ balance contrastive learning and environmental regularization.

Overall, the results demonstrate that while $\beta = 1$ universally preserves structure, each dataset requires a tailored balance of accuracy, fairness, and environmental regularization to achieve optimal performance.

\begin{table}[h]
\centering
\caption{Optimal HSCCAF Hyper-parameters per Dataset}\label{hyper}
\resizebox{0.40\textwidth}{!}{%
\begin{tabular}{|l|c|c|c|c|c|}
\hline
\textbf{Dataset} & $\alpha$ & $\beta$ & $\gamma$ & $\omega$ & $\eta$ \\
\hline
\textbf{German} & 10 & 1 & 1 & 0.3 & 0.09 \\
\hline
\textbf{Bail} & 0.2 & 1 & 0.02 & 0.03 & 0.07 \\
\hline
\textbf{Credit} & 0.5 & 1 & 1 & 0.7 & 0.06 \\
\hline
\textbf{NBA} & 0.9 & 1 & 1 & 0.09 & 0.8 \\
\hline
\textbf{Pokec-n} & 0.2 & 1 & 1 & 0.09 & 0.1 \\
\hline
\end{tabular}
}
\end{table}


In summary, HSCCAF effectively balances fairness and performance across different datasets.

\section{Conclusion}
In this paper, we proposed HSCCAF, an extension of the CAF model for fair node representation learning. HSCCAF integrates an editing strategy to mitigate topological bias, a supervised contrastive loss to enhance content learning, and an environmental loss to strengthen the association with sensitive attributes. Together, these components ensure effective disentanglement of content and sensitive information while improving predictive performance.  

Experimental results on multiple real-world datasets demonstrate that HSCCAF consistently enhances fairness 
while maintaining competitive accuracy,
outperforming CAF and other state-of-the-art baselines. Future work will explore extending HSCCAF to dynamic and heterogeneous graphs and integrating it with graph foundation models for broader applicability.
\subsection*{Limitations}
A key limitation of our approach, as well as that of the CAF paper, is its reliance on fully supervised sensitive attribute labels during training. In practice, access to complete and accurate sensitive attribute information is often limited or infeasible, especially in real-world scenarios with privacy constraints or incomplete annotations. While the environmental loss and supervised contrastive loss components in our framework do not require fully supervised sensitive attribute labels, other parts of the method still depend on them, which may reduce applicability in low-resource or partially labeled settings.
Another limitation is the inclusion of both contrastive and environmental losses, which introduces additional hyper-parameters that require careful tuning. Although our ablation study demonstrates that these components improve both fairness and accuracy, the need for hyper-parameter optimization may affect the method’s robustness and generalizability across different datasets and environments.
Finally, our method is limited to homophily-based graphs, restricting its applicability to graphs where connections primarily exist between similar nodes.


\begin{thebibliography}{10}

\bibitem{zhu2020beyond}
Jiong Zhu, Yujun Yan, Lingxiao Zhao, Mark Heimann, Leman Akoglu, and Danai Koutra.
\newblock Beyond homophily in graph neural networks: Current limitations and effective designs.
\newblock {\em Advances in neural information processing systems}, 33:7793--7804, 2020.

\bibitem{Gkarmpounis2024}
Georgios Gkarmpounis, Christos Vranis, Nicholas Vretos, and Petros Daras.
\newblock Survey on graph neural networks.
\newblock {\em IEEE Access}, 12:128816--128832, 2024.

\bibitem{chen-survey}
April Chen, Ryan~A. Rossi, Namyong Park, Puja Trivedi, Yu~Wang, Tong Yu, Sungchul Kim, Franck Dernoncourt, and Nesreen~K. Ahmed.
\newblock Fairness-aware graph neural networks: A survey.
\newblock {\em ACM Trans. Knowl. Discov. Data}, 18(6), April 2024.

\bibitem{chouldechova2020snapshot}
Alexandra Chouldechova and Aaron Roth.
\newblock A snapshot of the frontiers of fairness in machine learning.
\newblock {\em Communications of the ACM}, 63(5):82--89, 2020.

\bibitem{3495724.3497012}
Luca Oneto, Michele Donini, Giulia Luise, Carlo Ciliberto, Andreas Maurer, and Massimiliano Pontil.
\newblock Exploiting mmd and sinkhorn divergences for fair and transferable representation learning.
\newblock In {\em Proceedings of the 34th International Conference on Neural Information Processing Systems}, NIPS'20, Red Hook, NY, USA, 2020. Curran Associates Inc.

\bibitem{barocas2017fairness}
Solon Barocas, Moritz Hardt, and Arvind Narayanan.
\newblock Fairness in machine learning.
\newblock {\em Nips tutorial}, 1:2, 2017.

\bibitem{besse2019can}
Philippe Besse, C{\'e}line Castets-Renard, Aur{\'e}lien Garivier, and Jean-Michel Loubes.
\newblock Can everyday ai be ethical? machine learning algorithm fairness.
\newblock {\em Machine Learning Algorithm Fairness (May 20, 2018). Statistiques et Soci{\'e}t{\'e}}, 6(3), 2019.

\bibitem{besse2021survey}
Philippe Besse, Eustasio del Barrio, Paula Gordaliza, Jean-Michel Loubes, and Laurent Risser.
\newblock A survey of bias in machine learning through the prism of statistical parity.
\newblock {\em The American Statistician}, pages 1--11, 2021.

\bibitem{chen2024fairness}
April Chen, Ryan~A Rossi, Namyong Park, Puja Trivedi, Yu~Wang, Tong Yu, Sungchul Kim, Franck Dernoncourt, and Nesreen~K Ahmed.
\newblock Fairness-aware graph neural networks: A survey.
\newblock {\em ACM Transactions on Knowledge Discovery from Data}, 18(6):1--23, 2024.

\bibitem{spinelli2021fairdrop}
Indro Spinelli, Simone Scardapane, Amir Hussain, and Aurelio Uncini.
\newblock Fairdrop: Biased edge dropout for enhancing fairness in graph representation learning.
\newblock {\em IEEE Transactions on Artificial Intelligence}, 3(3):344--354, 2021.

\bibitem{liu2023generalized}
Zemin Liu, Trung-Kien Nguyen, and Yuan Fang.
\newblock On generalized degree fairness in graph neural networks.
\newblock In {\em Proceedings of the AAAI Conference on Artificial Intelligence}, volume~37, pages 4525--4533, 2023.

\bibitem{kose2023fairness}
O~Deniz Kose, Yanning Shen, and Gonzalo Mateos.
\newblock Fairness-aware graph filter design.
\newblock In {\em 2023 57th Asilomar Conference on Signals, Systems, and Computers}, pages 330--334. IEEE, 2023.

\bibitem{kusner2017counterfactua}
Matt~J Kusner, Joshua Loftus, Chris Russell, and Ricardo Silva.
\newblock Counterfactual fairness.
\newblock In {\em Advances in Neural Information Processing Systems}, volume~30, pages 4066--4076. Curran Associates, Inc., 2017.

\bibitem{de2021transport}
Lucas De~Lara, Alberto Gonz{\'a}lez-Sanz, Nicholas Asher, and Jean-Michel Loubes.
\newblock Transport-based counterfactual models.
\newblock {\em arXiv preprint arXiv:2108.13025}, 2021.

\bibitem{kang2020inform}
Jian Kang, Jingrui He, Ross Maciejewski, and Hanghang Tong.
\newblock Inform: Individual fairness on graph mining.
\newblock In {\em Proceedings of the 26th ACM SIGKDD international conference on knowledge discovery \& data mining}, pages 379--389, 2020.

\bibitem{agarwal2021towards}
Chirag Agarwal, Himabindu Lakkaraju, and Marinka Zitnik.
\newblock Towards a unified framework for fair and stable graph representation learning.
\newblock In {\em Uncertainty in Artificial Intelligence}, pages 2114--2124. PMLR, 2021.

\bibitem{ma2022learning}
Jing Ma, Ruocheng Guo, Mengting Wan, Longqi Yang, Aidong Zhang, and Jundong Li.
\newblock Learning fair node representations with graph counterfactual fairness.
\newblock In {\em Proceedings of the Fifteenth ACM International Conference on Web Search and Data Mining}, pages 695--703, 2022.

\bibitem{guo2023towards}
Zhimeng Guo, Jialiang Li, Teng Xiao, Yao Ma, and Suhang Wang.
\newblock Towards fair graph neural networks via graph counterfactual.
\newblock In {\em Proceedings of the 32nd ACM international conference on information and knowledge management}, pages 669--678, 2023.

\bibitem{li2024rethinking}
Zhixun Li, Yushun Dong, Qiang Liu, and Jeffrey~Xu Yu.
\newblock Rethinking fair graph neural networks from re-balancing.
\newblock In {\em Proceedings of the 30th ACM SIGKDD conference on knowledge discovery and data mining}, pages 1736--1745, 2024.

\bibitem{jiang2022topology}
Zhimeng Jiang, Xiaotian Han, Chao Fan, Zirui Liu, Xiao Huang, Na~Zou, Ali Mostafavi, and Xia Hu.
\newblock Topology matters in fair graph learning: a theoretical pilot study.(2022), 2022.

\bibitem{mcpherson2001birds}
Miller McPherson, Lynn Smith-Lovin, and James~M Cook.
\newblock Birds of a feather: Homophily in social networks.
\newblock {\em Annual review of sociology}, 27(1):415--444, 2001.

\bibitem{ciotti2016homophily}
Valerio Ciotti, Moreno Bonaventura, Vincenzo Nicosia, Pietro Panzarasa, and Vito Latora.
\newblock Homophily and missing links in citation networks.
\newblock {\em EPJ Data Science}, 5(1):7, 2016.

\bibitem{loveland2025unveiling}
Donald Loveland and Danai Koutra.
\newblock Unveiling the impact of local homophily on gnn fairness: In-depth analysis and new benchmarks.
\newblock In {\em Proceedings of the 2025 SIAM International Conference on Data Mining (SDM)}, pages 608--617. SIAM, 2025.

\bibitem{laclau2022survey}
Charlotte Laclau, Christine Largeron, and Manvi Choudhary.
\newblock A survey on fairness for machine learning on graphs.
\newblock {\em arXiv preprint arXiv:2205.05396}, 2022.

\bibitem{dai2021say}
Enyan Dai and Suhang Wang.
\newblock Say no to the discrimination: Learning fair graph neural networks with limited sensitive attribute information.
\newblock In {\em Proceedings of the 14th ACM International Conference on Web Search and Data Mining}, pages 680--688, 2021.

\bibitem{wang2022improving}
Yu~Wang, Yuying Zhao, Yushun Dong, Huiyuan Chen, Jundong Li, and Tyler Derr.
\newblock Improving fairness in graph neural networks via mitigating sensitive attribute leakage.
\newblock In {\em Proceedings of the 28th ACM SIGKDD conference on knowledge discovery and data mining}, pages 1938--1948, 2022.

\bibitem{rahman2019fairwalk}
Tahleen Rahman, Bartlomiej Surma, Michael Backes, and Yang Zhang.
\newblock Fairwalk: Towards fair graph embedding.
\newblock In {\em Proceedings of the Twenty-Eighth International Joint Conference on Artificial Intelligence, {IJCAI} 2019}, pages 3289--3295. IJCAI, 2019.

\bibitem{dong2022edits}
Yushun Dong, Ninghao Liu, Brian Jalaian, and Jundong Li.
\newblock Edits: Modeling and mitigating data bias for graph neural networks.
\newblock In {\em Proceedings of the ACM Web Conference 2022}, pages 1259--1269, 2022.

\bibitem{gordaliza2019obtaining}
Paula Gordaliza, Eustasio Del~Barrio, Gamboa Fabrice, and Jean-Michel Loubes.
\newblock Obtaining fairness using optimal transport theory.
\newblock In {\em International conference on machine learning}, pages 2357--2365. PMLR, 2019.

\bibitem{laclau2021all}
Charlotte Laclau, Ievgen Redko, Manvi Choudhary, and Christine Largeron.
\newblock All of the fairness for edge prediction with optimal transport.
\newblock In {\em International Conference on Artificial Intelligence and Statistics}, pages 1774--1782. PMLR, 2021.

\bibitem{zhu2024fair}
Yuchang Zhu, Jintang Li, Zibin Zheng, and Liang Chen.
\newblock Fair graph representation learning via sensitive attribute disentanglement.
\newblock In {\em Proceedings of the ACM Web Conference 2024}, pages 1182--1192, 2024.

\bibitem{liu5320563rethinking}
Chuxun Liu, Debo Cheng, Qingfeng Chen, Jiuyong Li, Lin Liu, and Rongyao Hu.
\newblock Rethinking fair graph representation learning via structural rebalancing.
\newblock {\em Available at SSRN 5320563}, 2025.

\bibitem{bondy2008graph}
John~Adrian Bondy and Uppaluri Siva~Ramachandra Murty.
\newblock {\em Graph theory}.
\newblock Springer Publishing Company, Incorporated, 2008.

\bibitem{kipf2016semi}
Thomas~N Kipf and Max Welling.
\newblock Semi-supervised classification with graph convolutional networks.
\newblock {\em arXiv preprint arXiv:1609.02907}, 2016.

\bibitem{hamilton2017inductive}
Will Hamilton, Zhitao Ying, and Jure Leskovec.
\newblock Inductive representation learning on large graphs.
\newblock {\em Advances in neural information processing systems}, 30, 2017.

\bibitem{chen2023fairness}
April Chen, Ryan~A Rossi, Namyong Park, Puja Trivedi, Yu~Wang, Tong Yu, Sungchul Kim, Franck Dernoncourt, and Nesreen~K Ahmed.
\newblock Fairness-aware graph neural networks: A survey.
\newblock {\em ACM Transactions on Knowledge Discovery from Data}, 2023.

\bibitem{khosla2020supervised}
Prannay Khosla, Piotr Teterwak, Chen Wang, Aaron Sarna, Yonglong Tian, Phillip Isola, Aaron Maschinot, Ce~Liu, and Dilip Krishnan.
\newblock Supervised contrastive learning.
\newblock {\em Advances in neural information processing systems}, 33:18661--18673, 2020.

\bibitem{kobayashi2021t}
Takumi Kobayashi.
\newblock T-vmf similarity for regularizing intra-class feature distribution.
\newblock In {\em Proceedings of the IEEE/CVF conference on computer vision and pattern recognition}, pages 6616--6625, 2021.

\bibitem{asuncion2007uci}
Arthur Asuncion and David Newman.
\newblock Uci machine learning repository, 2007.

\bibitem{jordan2015effect}
Kareem~L Jordan and Tina~L Freiburger.
\newblock The effect of race/ethnicity on sentencing: Examining sentence type, jail length, and prison length.
\newblock {\em Journal of Ethnicity in Criminal Justice}, 13(3):179--196, 2015.

\bibitem{yeh2009comparisons}
I-Cheng Yeh and Che-hui Lien.
\newblock The comparisons of data mining techniques for the predictive accuracy of probability of default of credit card clients.
\newblock {\em Expert systems with applications}, 36(2):2473--2480, 2009.

\bibitem{takac2012data}
Lubos Takac and Michal Zabovsky.
\newblock Data analysis in public social networks.
\newblock In {\em International scientific conference and international workshop present day trends of innovations}, volume~1, 2012.

\bibitem{xu2019powerful}
Keyulu Xu, Weihua Hu, Jure Leskovec, and Stefanie Jegelka.
\newblock How powerful are graph neural networks?, 2019.

\bibitem{10.1145/3437963.3441752}
Enyan Dai and Suhang Wang.
\newblock Say no to the discrimination: Learning fair graph neural networks with limited sensitive attribute information.
\newblock In {\em Proceedings of the 14th ACM International Conference on Web Search and Data Mining}, WSDM '21, page 680–688, New York, NY, USA, 2021. Association for Computing Machinery.

\end{thebibliography}

\begin{appendices}
\section{Proof of Properties of Fairness-Aware Graph Editing}
\label{appendix}

Let \( G = (\mathcal{V}, \mathcal{E}) \) be a finite, simple, undirected graph with \( |\mathcal{V}| = n \) and \( |\mathcal{E}| = n \geq 1\). 
Each node \(v \in \mathcal{V}\) carries a binary class label  \(y(v) \in\{0,1\}\) and 
a binary sensitive attribute  \(s(v) \in\{0,1\}\).

Define the number of class-homophilous and sensitive-homophilous edges as
\begin{equation}
N_c = \bigl|\{\,\{u,v\}\in \mathcal{E}:\ y(u)=y(v)\,\}\bigr| 
\end{equation}
\begin{equation}
N_s = \bigl|\{\,\{u,v\}\in \mathcal{E}:\ s(u)=s(v)\,\}\bigr|
\end{equation}

The global edge-homophily ratios are then
\begin{equation}
hr_{c} = \frac{N_c}{m} \label{eq:hc}
\end{equation}
\begin{equation}
hr_{s} = \frac{N_s}{m} \label{eq:hs}
\end{equation}

Finally, partition edges into four types depending on class and sensitive attribute:

\vspace{0.5em}
\begin{tabular}{@{}l l l@{}}
\textbf{Type I:}   & $y(u) = y(v),$     & $s(u) = s(v)$ \\
\textbf{Type II:}  & $y(u) = y(v),$     & $s(u) \neq s(v)$ \\
\textbf{Type III:} & $y(u) \neq y(v),$  & $s(u) = s(v)$ \\
\textbf{Type IV:}  & $y(u) \neq y(v),$  & $s(u) \neq s(v)$
\end{tabular}
\vspace{0.5em}

\paragraph{Operation (FairEdgeRemove)}
Let $K\subseteq \mathcal{E}$ be any set consisting only of Type~III edges. Form $G'=(\mathcal{V},\mathcal{E'})$ with $\mathcal{E'}=\mathcal{E}\setminus K$, and denote $m'=|\mathcal{E'}|=m-|K|$, $N_c'$, $N_s'$, and the updated ratios $hr_c'=N_c'/m'$, $hr_s'=N_s'/m'$.

\begin{theorem}[Monotonic effect of removing Type~III edges]
\label{thm:monotone}
For any $K$ consisting only of Type~III edges:
\begin{equation}
hr_c'\ \ge\ hr_c
\qquad\text{and}\qquad
hr_s'\ \le\ hr_s.
\end{equation}
Moreover, the changes satisfy the exact identities
\begin{equation}
hr'_c-hr_c=\frac{N_c\,|K|}{m(m-|K|)}\ge\ 0
\end{equation}
\begin{equation}
hr'_s-hr_s=\frac{|K|\,(N_s-m)}{m(m-|K|)}\ \le\ 0,
\end{equation}
with equality if and only if $|K|=0$ or $N_s=m$.
\end{theorem}

\begin{proof}
Type~III edges contribute to the denominator $m$ but not to $N_c$. Thus $N_c'=N_c$ and $m'=m-|K|$, which yields
\begin{equation}
hr'_c-hr_c=\frac{N_c}{m-|K|}-\frac{N_c}{m}=\frac{N_c|K|}{m(m-|K|)}\ge 0.
\end{equation}

Type~III edges \emph{do} contribute to $N_s$. Hence $N_s'=N_s-|K|$ and
\begin{align}
hr'_s - hr_s
&= \frac{N_s - |K|}{m - |K|}
   - \frac{N_s}{m} \notag \\
&= \frac{m(N_s - |K|) - (m - |K|)N_s}{m(m - |K|)} \notag \\
&= \frac{|K|(N_s - m)}{m(m - |K|)} \le 0.
\end{align}

since $N_s\le m$. The equalities are immediate from the displayed formulas.
\end{proof}

\begin{proposition}[Single-edge type optimality under deletions] \label{prop:optimal}

Fix an integer $k\in\{1,\dots,m\}$. Among all choices of $k$ edges to delete:
\begin{itemize}
\item The maximal possible increase in $hr_c$ is achieved by deleting only Type~III or Type~IV edges, and the \emph{simultaneous} decrease of $hr_s$ selects Type~III as the unique edge type that increases $hr_c$ while decreasing $hr_s$.
\item The maximum possible decrease in $h_s$ is achieved by deleting only Type~I or Type~III edges, and the \emph{simultaneous} increase of $h_c$ selects Type~III as the unique edge type that decreases $hr_s$ while increasing $hr_c$.
\end{itemize}
Consequently, if the goal is to \emph{increase} $hr_c$ and \emph{decrease} $hr_s$ using exactly $k$ deletions, any optimal solution deletes $k$ Type~III edges whenever at least $k$ such edges exist.
2\end{proposition}

\begin{proof}
Consider the per-edge effect on $(hr_c,hr_s)$ for each type using the same algebra as in Theorem~\ref{thm:monotone} with $|K|=1$:
\begin{equation}
\Delta hr_c =
\begin{cases}
\frac{N_c-m}{m(m-1)}<0, & \text{Type I}\\[2pt]
\frac{N_c}{m(m-1)}>0, & \text{Type II}\\[2pt]
\frac{N_c}{m(m-1)}>0, & \text{Type III}\\[2pt]
\frac{N_c}{m(m-1)}>0, & \text{Type IV}
\end{cases}
\end{equation}
\begin{equation}
\Delta hr_s =
\begin{cases}
\frac{N_s-m}{m(m-1)}\le 0, & \text{Type I}\\[2pt]
\frac{N_s}{m(m-1)}\ge 0, & \text{Type II}\\[2pt]
\frac{N_s-m}{m(m-1)}\le 0, & \text{Type III}\\[2pt]
\frac{N_s}{m(m-1)}\ge 0, & \text{Type IV}
\end{cases}
\end{equation}
Thus only Type~III simultaneously yields $\Delta hr_c>0$ and $\Delta hr_s<0$. Linearity of the numerators in batch deletions and the fact that each deletion updates only counts entering $N_c,N_s,m$ imply that any set with the desired monotone effects must be a subset of Type~III. If at least $k$ Type~III edges exist, deleting any $k$ of them achieves both targets, while any mixture including Types I, II, or IV fails at least one target. This yields the claims.
\end{proof}

\begin{corollary}[Minimal deletions to reach targets]
\label{cor:min-k}
Suppose targets $\tau_c\in(hr_c,1]$ and $\tau_s\in[0,hr_s)$ are given. Let $M_{\mathrm{III}}$ be the number of Type~III edges. Define
\begin{equation}
k_c^\star=\min\bigl\{k\in\{0,\dots,M_{\mathrm{III}}\}: \tfrac{N_c}{m-k}\ge \tau_c\bigr\},
\end{equation}
\begin{equation}
k_s^\star=\min\bigl\{k\in\{0,\dots,M_{\mathrm{III}}\}: \tfrac{N_s-k}{m-k}\le \tau_s\bigr\}.
\end{equation}
If $k^\star=\max\{k_c^\star,k_s^\star\}\le M_{\mathrm{III}}$, then deleting any $k^\star$ Type~III edges attains both targets with the minimum number of deletions.
\end{corollary}
\begin{proof}
By Theorem~\ref{thm:monotone}, $hr_c(N)$ is nondecreasing and $hr_s(N)$ is nonincreasing in the number $N$ of deleted Type~III edges, with explicit forms $h_c(N)=N_c/(m-N)$ and $h_s(N)=(N_s-N)/(m-N)$. The smallest $N$ achieving each target is given by the displayed thresholds; their maximum achieves both simultaneously. Proposition~\ref{prop:optimal} shows optimality within all $k$-deletion strategies.\end{proof}
\end{appendices}
\end{document}